\documentclass{article} 
\usepackage{nips14submit_e,times}
\usepackage{url}

\usepackage{amsmath}
\usepackage{amssymb}
\usepackage{amsfonts}
\usepackage{amsthm}
\newtheorem{thm}{Theorem}[section]
\newtheorem{lem}[thm]{Lemma}
\newtheorem{cor}[thm]{Corollary}

\theoremstyle{definition}
\newtheorem{defn}{Definition}
\newtheorem*{remark}{Remark}

\usepackage{graphicx}
\usepackage{subcaption}
\usepackage{url}
\usepackage[font=small]{caption}

\usepackage{pdfpages}

\renewcommand{\epsilon}{\varepsilon}
\newcommand{\eps}{\varepsilon}
\newcommand{\beps}{\overline{\eps}}
\newcommand{\RR}{\mathbb{R}}
\newcommand{\cX}{\mathcal{X}}
\newcommand{\NB}{\mathsf{NB}}
\newcommand{\cD}{\mathsf{D}}

\usepackage{xr}
\usepackage{prettyref}
\newrefformat{Sthm}{Theorem~S\ref{#1}}
\newrefformat{Slem}{Lemma~S\ref{#1}}
\newrefformat{Scor}{Corollary~S\ref{#1}}
\newrefformat{Sconj}{Conjecture~S\ref{#1}}

\usepackage{placeins}

\title{Metric recovery from directed unweighted graphs}

\author{
Tatsunori B. Hashimoto\\
MIT CSAIL\\
Cambridge, MA 02139 \\
\texttt{thashim@csail.mit.edu} \\
\And
Yi Sun \\
MIT Dept. Mathematics \\
Cambridge, MA 02139 \\
\texttt{yisun@math.mit.edu} \\
\AND
Tommi S. Jaakkola \\
MIT CSAIL \\
Cambridge, MA 02139 \\
\texttt{tommi@csail.mit.edu} \\
}

\nipsfinalcopy 

\begin{document}

\maketitle

\begin{abstract}

We analyze directed, unweighted graphs obtained from $x_i\in \RR^d$ by connecting vertex $i$ to $j$ iff $|x_i - x_j| < \epsilon(x_i)$. Examples of such graphs include $k$-nearest neighbor graphs, where $\epsilon(x_i)$ varies from point to point, and, arguably, many real world graphs such as co-purchasing graphs. We ask whether
we can recover the underlying Euclidean metric $\epsilon(x_i)$ and the associated density $p(x_i)$ given only the directed graph and $d$.

We show that consistent recovery is possible up to isometric scaling when the vertex degree is at least $\omega(n^{2/(2+d)}\log(n)^{d/(d+2)})$. Our estimator is based on a careful characterization of a random walk over the directed graph and the associated continuum limit. As an algorithm, it resembles the PageRank centrality metric. We demonstrate empirically that the estimator performs well on simulated examples as well as on real-world co-purchasing graphs even with a small number of points and degree scaling as low as $\log(n)$.
\end{abstract}

\section{Introduction}

Data for unsupervised learning is increasingly available in the form of graphs or networks. For example, we may analyze gene networks, social networks, or general co-occurrence graphs (e.g., built from purchasing patterns). While classical unsupervised tasks such as density estimation or clustering are naturally formulated for data in vector spaces, these tasks have analogous problems over graphs such as centrality and community detection. We provide a step towards unifying unsupervised learning by recovering the underlying density and metric directly from graphs. 

We consider ``unweighted directed geometric graphs'' that are assumed to have been built from underlying (unobserved) points $x_i$, $i=1,\ldots,n$. In particular, we assume that graphs are formed by drawing an arc from each vertex $i$ to its neighbors within distance $\epsilon_n(x_i)$. Note that the graphs are typically not symmetric since the distance (the $\epsilon_n$-ball) may vary from point to point. By allowing $\epsilon_n(x_i)$ to be stochastic, e.g., depend on the set of points, the construction subsumes also typical $k$-nearest neighbor graphs. Arguably, graphs from top $k$ friends/products, or co-association graphs may also be approximated in this manner. 

The key property of our family of geometric graphs is that their structure is completely characterized by two functions over the latent space: the local density $p(x)$ and the local scale $\epsilon(x)$. Indeed, global properties such as the distances between points can be recovered by integrating these quantities. We show that asymptotic behavior of random walks on the directed graphs relate to the density and metric. In particular, we show that random walks on such graphs with minimal degree at least $\omega(n^{2/(2+d)}\log(n)^{\frac{d}{d+2}})$ can be completely characterized in terms of $p$ and $\epsilon$ using drift-diffusion processes. This enables us to recover both the density and distance given only the observed graph and the (hypothesized) underlying dimension $d$. 

The fact that we may recover the density (up to isometry) is surprising. For example, in $k$-nearest neighbor graphs, each vertex has degree exactly $k$. There is no immediate local information about the density, i.e., whether the corresponding point lies in a high-density region with small ball radii, or in a low-density region with large ball radii. The key insight of this paper is that random walks over such graphs naturally drift toward higher density regions, allowing for density recovery.

While the paper is primarily focused on the theoretical aspects of recovering the metric and density, we believe our results offer useful strategies for analyzing real-world networks. For example, we analyzed the Amazon co-purchasing graph where an edge is drawn from an item $i$ to $j$ if $j$ is among the top $k$ co-purchased items with $i$. These Amazon products may be co-purchased if they are similar enough to be complementary, but not so similar that they are redundant. We extend our model to deal with connectivity rules shaped like an annulus, and demonstrate that our estimator can simultaneously recover product similarities, product categories, and central products by metric embedding.

\subsection{Relation to prior work}

The density estimation problem addressed by this paper was proposed and partially solved by von Luxburg-Alamgir in \cite{alamgir_density_2013} using integration of local density gradients over shortest paths. This estimator has since been used for drawing graphs with ordinal constraints in \cite{alamgir_density_2013} and graph down-sampling in \cite{alamgir_density-preserving_2014}. However, the recovery algorithm is restricted to $1$-dimensional $k$-nearest neighbor graphs under the constraint $k=\omega(n^{2/3}\log(n)^{\frac{1}{3}})$. Our paper provides an estimator that works in all dimensions, applies to a more general class of graphs, and strongly outperforms that of von Luxburg-Alamgir in practice.

On a technical level, our work has similarities to the analysis of convergence of graph Laplacians and random walks on manifolds in \cite{woess_random_1994, hein_graph_2006}. For example, in \cite{ting2010analysis}, Ting-Huang-Jordan used infinitesimal generators to capture the convergence of a discrete Laplacian to its continuous equivalent on $k$-nearest neighbor graphs. However, their analysis was restricted to the Laplacian and did not consider the latent recovery problem.  In addition, our approach proves convergence of the entire random walk trajectory and allows us to analyze the stationary distribution function directly. 


\section{Main results and proof outline}

\subsection{Problem setup}

Let $\cX = \{x_1, x_2, \ldots\}$ be an infinite sequence of latent coordinate points drawn independently from a distribution with probability density $p(x)$ in $\RR^d$.  Let $\eps_n(x_i)$ be a radius function which may depend on the draw of $\cX$.  In this paper, we fix a single draw of $\cX$ and analyze the quenched setting.  Let $G_n = (\cX_n, E_n)$ be the unweighted directed neighborhood graph with vertex set $\cX_n = \{x_1, \ldots, x_n\}$ and with a directed edge from $i$ to $j$ if and only if $|x_i - x_j| < \eps_n(x_i)$.

Fix now a large $n$.  We consider the random directed graph model given by observing the single graph $G_n$.  The model is completely specified by the latent function $p(x)$ and the possibly stochastic $\eps_n(x)$.  Under the conditions ($\star$) to be specified below, we solve the following problem:

\newenvironment{myquote}{\list{}{\leftmargin=0.1in\rightmargin=0.1in}\item[]}{\endlist}
\begin{myquote}
Given only $G_n$ and $d$, form a consistent estimate of $p(x_i)$ and $|x_i - x_j|$ up to proportionality constants.
\end{myquote}

The conditions we impose on $p(x)$, $\eps_n(x)$, and the stationary density function $\pi_{X_n}(x)$ of the simple random walk $X_n(t)$ on $G_n$ are the following, which we refer to as ($\star$).  We assume ($\star$) holds throughout the paper.
\begin{itemize}
\item The density $p(x)$ is differentiable with bounded $\nabla \log(p(x))$ on a path-connected compact domain $D \subset \RR^d$ with smooth boundary $\partial D$.
\item There is a deterministic continuous function $\beps(x) > 0$ on $\overline{D}$  and scaling constants $g_n$ satisfying 
\[
g_n \to 0 \text{  and  } g_n n^{\frac{1}{d + 2}}\log(n)^{-\frac{1}{d + 2}} \to \infty
\]
so that, a.s. in the draw of $\cX$, $g_n^{-1} \eps_n(x)$ converges uniformly to $\beps(x)$.
\item The rescaled density functions $n \pi_{X_n}(x)$ are a.s. uniformly equicontinuous.
\end{itemize}

\begin{remark}
 We conjecture that the last condition in ($\star$) holds for any $p$ and $\beps$ satisfying the other conditions in ($\star$) (see \prettyref{Sconj:holder}).
\end{remark}

Let $\NB_n(x)$ denote the set of out-neighbors of $x$ so that $y$ is in $\NB_n(x)$ if there is a directed edge from $x$ to $y$. The second condition in ($\star$) implies for all $x \in \cX_n$ that
\begin{equation} \label{eq:nb}
|\NB_n(x)| = \omega(n^{\frac{2}{d + 2}} \log(n)^{\frac{d}{d+2}}).
\end{equation}

\subsection{Statement of results}

Our approach is based on the simple random walk $X_n(t)$ on the graph $G_n$.  Let $\pi_{X_n}(x)$ denote the stationary density of $X_n(t)$.  We first show that when appropriately renormalized, $\pi_{X_n}(x)$ converges to an explicit function of $p(x)$ and $\beps(x)$.

\begin{thm} \label{thm:mainthm}
Given ($\star$), a.s. in $\cX$, we have
\begin{equation} \label{eq:maineq}
n\pi_{X_n}(x) \rightarrow c \frac{p(x)}{\overline{\epsilon}(x)^2},
\end{equation}
for the normalization constant $c^{-1} = \int p(x)^2 \beps(x)^{-2}dx$.
\end{thm}

Combining this result with an estimate on the out-degree of points in $G_n$ gives our general result on recovery of density and scale. Let $V_d$ be the volume of the unit $d$-ball.

\begin{cor} \label{cor:maincor}
Assuming ($\star$), we have a.s. in $\cX$ that
\begin{align*}
\left(\frac{n^{\frac{d - 2}{d}}}{cV_d^{2/d}g_n^2}\right)^{\frac{d}{d+2}} |\NB_n(x)|^{\frac{2}{d+2}} \pi_{X_n}(x)^{\frac{d}{d + 2}} &\rightarrow p(x) \text{ and}\\
\left(\frac{1}{c^{d/2} V_d n^{2} g_n^d}\right)^{\frac{1}{d+2}}  |\NB_n(x)|^{\frac{1}{d+2}} \pi_{X_n}(x)^{-\frac{1}{d+2}} &\to \beps(x).
\end{align*}
\end{cor}
\begin{proof}
Immediate from the out-degree estimate $p(x)\epsilon_n(x)^dV_d = |\NB_n(x)|/n$ and Theorem \ref{thm:mainthm}.
\end{proof}

\begin{remark}
If $\epsilon_n(x)$ is constant, every edge is bidirectional, so $\pi_{X_n}(x)$ is proportional to the degree of $x$, and we recover the standard $\epsilon$-ball density estimator.
\end{remark}

Our estimator for density $p(x)$ closely resembles the PageRank algorithm without damping \cite{page1999pagerank}. In particular, for the $k$-nearest neighbor graph, it gives the same rank ordering as PageRank, and it reduces to PageRank as $d \to \infty$.

When specializing to the $k$-nearest neighbor density estimation problem posed by von Luxburg-Alamgir in \cite{alamgir_density_2013},  we obtain the following.
\begin{cor}
If $\epsilon_n(x)$ is selected via the $k$-nearest neighbors procedure with $k = \omega(n^{\frac{2}{d+2}} \log(n)^{\frac{d}{d+2}})$ and satisfies the first and last conditions in ($\star$), we have a.s in $\cX$ that
\begin{align*}
\left(\frac{n}{cV_d^{2/d}} \right)^{\frac{d}{d+2}} \pi_{X_n}(x)^{\frac{d}{d + 2}} &\rightarrow p(x) \text{ and}\\
\left(\frac{1}{c^{d/2} V_d n} \right)^{\frac{1}{d+2}}\pi_{X_n}(x)^{-\frac{1}{d + 2}} &\rightarrow \overline{\epsilon}(x).
\end{align*}
\end{cor}
\begin{proof}
By \cite{devroye1977strong}, the empirical $\eps_n(x)$ induced by the $k$-nearest neighbors procedure satisfies the second condition of ($\star$) with 
\[
\beps(x) = \frac{1}{V_d^{1/d} p(x)^{1/d}} \text{ and } g_n = (k/n)^{1/d}. \qedhere
\]
\end{proof}

\subsection{Outline of approach}

Our proof proceeds via the following steps.

\begin{enumerate}
 \setlength{\itemsep}{-1pt}
\item[1.] As $n \to \infty$, the simple random walks $X_n(t)$ on $G_n$ converge weakly to an It\^o process $Y(t)$, yielding weak convergence between stationary measures. (Theorem \ref{thm:converge})
\item[2.] The stationary density $\pi_Y(x)$ is explicitly determined via Fokker-Planck equation. (Lemma \ref{lem:fokker})
\item[3.] Uniform equicontinuity of $n\pi_{X_n}(x)$ yields convergence in density after rescaling. (Theorem \ref{thm:mainthm})
\end{enumerate}

An intuitive explanation for our results is as follows.  For large $n$, the simple random walk on $G_n$, when considered with its original metric embedding, closely approximates the behavior of a drift-diffusion process.  Both the process and the approximating walk move preferentially toward regions where $p(x)$ is large and diffuse more slowly out of regions where $\beps(x)$ is small.  Occupation times therefore give us information about $p(x)$ and $\beps(x)$ which allow us to recover them.

Formally, the convergence of $X_n(t)$ to $Y(t)$ follows by verifying the conditions of the Stroock-Varadhan criterion (Theorem \ref{thm:stroock}) for convergence of discrete time Markov processes to It\^o processes \cite{stroock_varadhan}.  This criterion states that if the variance $a_n$, expected value $b_n$, a higher order moments $\Delta_{n, \alpha}$ of a jump are continuous and well-controlled in the limit, then the process converges to an It\^o process under mild technical conditions.  By using the Fokker-Planck equation, we can express the stationary density of this It\^o process solely in terms of $p(x)$ and the out-degree $|\NB_n(x)|$.  This allows us to estimate the density using only the unweighted graph.

Let $\overline{D}$ and $\partial D$ be the closure and boundary of the support $D$ of $p(x)$.  Let $B(x, \epsilon)$ be the ball of radius $\epsilon$ centered at $x$.  Let $h_n = g_n^2$ be the time rescaling necessary for $X_n(t)$ to have timescale equal to that of $Y(t)$.


\section{Convergence of the simple random walk to an It\^o process}

We will verify the regularity conditions of the Stroock-Varadhan criterion (see \cite[Section 6]{stroock_varadhan}).

\begin{thm}[Stroock-Varadhan] \label{thm:sv}\label{thm:stroock}
Let $X_n(t)$ be discrete-time Markov processes defined over a domain $D$ with boundary $\partial D$. Define the discrete time drift and diffusion coefficients by
\begin{align*}
a_n^{ij}(s, x) &= \frac{1}{h_n} \sum_{y\in \NB_n(x)} \frac{1}{|\NB_n(x)|} (y_i - x_i)(y_j - x_j)\\
b_n^{i}(s, x) &= \frac{1}{h_n} \sum_{y \in \NB_n(x)} \frac{1}{|\NB_n(x)|} (y_i-x_i)\\
\Delta_{n, \alpha}(s, x) &= \frac{1}{h_n} \sum_{y \in \NB_n(x)} \frac{1}{|\NB_n(x)|} |y-x|^{2+\alpha}.
\end{align*}
If we have $a_n^{ij}(s,x) \xrightarrow{a.s} a^{ij}(s,x)$, $b_n^{i}(s,x) \xrightarrow{a.s} b^{i}(s,x)$, $\Delta_{n, 1}(s,x) \xrightarrow{a.s} 0$, and regularity conditions to ensure reflection at $\partial D$ (\prettyref{Sthm:tightness} and \prettyref{Sthm:stroock}), the time-rescaled stochastic processes $X_n(\lfloor t/h_n \rfloor)$ converge weakly in Skorokhod space $\cD([0, \infty), \overline{D})$ to an It\^o process with reflecting boundary condition
\[
dY(t) = \sigma(t,Y(t)) dW_t + b(t,Y(t)) dt,
\]
with $W_t$ a standard $d$-dimensional Brownian motion and $\sigma(t, Y(t)) \sigma(t, Y(t))^T = a(t, Y(t))$.
\end{thm}

\begin{remark}
The original result of Stroock-Varadhan was stated for $\cD([0, T], \overline{D})$ for all finite $T$; our version for $\cD([0, \infty), \overline{D})$ is equivalent by \cite[Theorem 2.8]{whitt}.
\end{remark}

The technical conditions of Theorem \ref{thm:sv} enforcing reflecting boundary conditions are checked in \prettyref{Sthm:C} to \prettyref{Sthm:B}. We focus on convergence of the drift and diffusion coefficients.

\begin{lem}[Strong LLN for local moments] \label{lem:lln}
For a function $f(x)$ such that $\sup_{x\in B(0,\eps)} |f(x)| < \eps$, given ($\star$) we have uniformly on $x \in \cX_n$ that
\begin{align*}
&\frac{1}{h_n} \sum_{y \in \NB_n(x)} \frac{1}{|\NB_n(x)|} f(y-x) \\
&\qquad\xrightarrow{a.s.} \frac{1}{h_n} \int_{y\in B(x, \eps_n(x))}f(y-x) \frac{p(y)}{p_{\eps_n(x)}(x)} dy.
\end{align*}
\end{lem}
\begin{proof}
Denote the claimed value of the limit by $\mu(x)$.  For convergence in expectation, we condition on $|\NB_n(x)|$ and apply iterated expectation to get
\begin{align*}
&E\left[\frac{1}{h_n} \sum_{y \in \NB_n(x)} \frac{1}{|\NB_n(x)|} f(y-x)\right]\\
&\qquad= E\left[\frac{1}{h_n} E\left[f(y-x) \big| |\NB_n(x)|\right]\right] = \mu(x).
\end{align*}
For $y \in B(x, \eps_n(x))$, we have $|f(y - x)| \leq \eps_n(x)$, so Hoeffding's inequality yields
\begin{align} \label{eq:lln}
P\bigg( \bigg| \frac{1}{h_n} &\sum_{y\in \NB_n(x)} \frac{1}{|\NB_n(x)|} f(y-x) - \mu(x) \bigg| \geq t \bigg) \nonumber \\
&\leq 2 \exp \left( - \frac{2 h_n^2 |\NB_n(x)|^2 t^2 }{|\NB_n(x)| \eps_n(x)^2 }\right) \nonumber \\
&= \Theta\left(\exp \left( - 2g_n^2 \beps(x)^{-2} |\NB_n(x)| t^2\right) \right)\\
&= o\left(n^{-\frac{2 p(x)^{2/d}t^2}{\beps(x)^4} \omega(1)}\right) \nonumber\\
&= o(n^{-2 t^2 \omega(1)}) \nonumber
\end{align}
for $|\NB_n(x)| = \omega\left( n^{2/(d+2)} \log(n)^{d/(d+2)} \right)$ by (\ref{eq:nb}).  Borel-Cantelli then yields a.s. convergence.
\end{proof}

\begin{remark}
This limit holds even for stochastic $\eps_n(x)$ as long as $g_n^{-1}\eps_n(x)$ a.s. converges uniformly to a deterministic continuous $\beps(x)$. All statements up to \prettyref{eq:lln} hold regardless of stochasticity of $\eps_n(x)$ and the overall bound only requires convergence of $\eps_n(x)$. An example of such a graph is the $k$-nearest neighbors graph.
\end{remark}

We now compute the drift and diffusion coefficients in terms of $p(x)$ and $\beps(x)$.

\begin{thm}[Drift diffusion coefficients] \label{thm:coefs}
Almost surely on the draw of $\cX$, as $n \to \infty$, we have
\begin{align*}
\lim_{n \to \infty} a^{ij}_n(s, x) &= \delta_{ij} \frac{1}{3} \overline{\eps}(x)^2 \\
\lim_{n \to \infty} b^i_n(s, x) &= \frac{\partial_i p(x)}{3p(x)} \overline{\eps}(x)^2\\
\lim_{n \to \infty} \Delta_{n, 1}(s, x) &= 0,
\end{align*}
where $\delta_{ij}$ is the Kronecker delta function.
\end{thm}
\begin{proof}
By Lemma \ref{lem:lln}, $a_n$, $b_n$, and $\Delta_{n, 1}$ converge a.s. to their expectations, so it suffices to verify that the integrals in Lemma \ref{lem:lln} have the claimed limits.  Because $p$ is differentiable on $D$, for any $x \in D$ we have the Taylor expansion
\[
p(x + y) = p(x) + y \cdot \nabla p(x) + o(|y|^2)
\]
of $p$ at $x$, where the convergence is uniform on compact sets.  For $n$ large so that $B(x, \eps_n(x))$ lies completely inside $D$, substituting this expansion into the definitions of $a_n$, $b_n$, and $\Delta_{n, 1}$ and integrating over spheres yields the result.  Full details are in \prettyref{Sthm:coefs}.
\end{proof}

\begin{thm} \label{thm:converge}
Under ($\star$), as $n\rightarrow \infty$ a.s. in the draw of $\cX$ the process $X_{n}(\lfloor t/h_n \rfloor)$ converges in $\cD([0, \infty), \overline{D})$ to the isotropic $\overline{D}$-valued It\^o process $Y(t)$ with reflecting boundary condition defined by
\begin{equation} \label{eq:ydef}
dY(t) = \frac{\nabla p(Y(t))}{3p(Y(t))}\overline{\epsilon}(Y(t))^2 dt + \frac{\overline{\epsilon}(Y(t))}{\sqrt{3}}  dW(t).
\end{equation}
\end{thm}
\begin{proof}
Lemma \ref{lem:lln} and Theorem \ref{thm:coefs} show that $X_n(\lfloor t/h_n\rfloor)$ fulfills the conditions of Theorem \ref{thm:stroock}.  The result follows from the Stroock-Varadhan criterion using the drift and diffusion terms from Theorem \ref{thm:coefs}.
\end{proof}

\section{Convergence and computation of the stationary distribution}

\subsection{Graphs satisfying condition ($\star$)}

The It\^o process $Y(t)$ is an isotropic drift-diffusion process, so the Fokker-Planck equation \cite{risken1984fokker} implies its density $f(t, x)$ at time $t$ satisfies
\begin{equation} \label{eq:fp}
\partial_t f(t, x) = \sum_i \bigg(- \partial_{x_i} [b^i(t, x)f(t, x)] +\frac{1}{2}\partial_{x_i^2}[a^{ii}(t, x)f(t, x)]\bigg),
\end{equation}
where $b^i(t, x)$ and $a^{ii}(t, x)$ are given by
\[
b(t, x) = \frac{\nabla p(x)}{3p(x)} \bar{\epsilon}(x)^2  \text{ and } a^{ii}(t, x) = \frac{1}{3} \bar{\epsilon}(x)^2.
\]

\begin{lem} \label{lem:fokker}
The process $Y(t)$ defined by (\ref{eq:ydef}) has absolutely continuous stationary measure with density
\[
\pi_Y(x) = cp(x)^2\overline{\eps}(x)^{-2},
\]
where $c$ was defined in (\ref{eq:maineq}).
\end{lem}
\begin{proof}
By (\ref{eq:fp}), to check that $\pi_Y(x) = cp(x)^2\overline{\epsilon}(x)^{-2}$, it suffices to show
\begin{equation*}
\sum_i \bigg(\partial_{x_i} p(x) \left(p(x)^{-1} \overline{\epsilon}(x)^{2}  c\frac{p(x)^2}{\overline{\epsilon}(x)^{2}}\right) - \frac{1}{2}\partial_{x_i}\left(\overline{\epsilon}(x)^{2} c \frac{p(x)^2}{\overline{\epsilon}(x)^2}\right)\bigg) = 0.\qedhere
\end{equation*}
\end{proof}

We now prove Theorem \ref{thm:mainthm} by showing that a rescaling of $\pi_{X_n}(x)$ converges to $\pi_{Y}(x)$.

\begin{proof}[Proof of Theorem \ref{thm:mainthm}]
The a.s.~convergence of processes of Theorem \ref{thm:converge} implies by Ethier-Kurtz \cite[Theorem 4.9.12]{ethier} that the empirical stationary measures
\[
d\mu_n = \sum_{i = 1}^n \pi_{X_n}(x_i) \delta_{x_i}
\]
converge weakly to the stationary measure $d\mu = \pi_Y(x) dx$ for $Y(t)$.  For any $x \in \cX$ and $\delta > 0$, weak convergence against $1_{B(x, \delta)}$ yields
\[
\sum_{y \in \cX_n, |y - x| < \delta} \pi_{X_n}(y) \to \int_{|y - x| < \delta} \pi_Y(y) dy.
\]
By uniform equicontinuity of $n \pi_{X_n}(x)$, for any $\eps > 0$ there is small enough $\delta > 0$ so that for all $n$ we have
\begin{equation*}
\left| \sum_{y \in \cX_n, |y - x| < \delta}\!\!\!\! \pi_{X_n}(y) - |\cX_n \cap B(x, \delta)| \pi_{X_n}(x)\right| \leq n^{-1}|\cX_n \cap B(x, \delta)| \eps,
\end{equation*}
which implies that
\begin{align*}
\lim_{n \to \infty} &\pi_{X_n}(x)p(x)n \\ &= \lim_{\delta \to 0} \lim_{n \to \infty}V_d^{-1}\delta^{-d}n \pi_{X_n}(x) \int_{|y - x| < \delta} p(y) dy \\
&= \lim_{\delta \to 0} \lim_{n \to \infty}V_d^{-1}\delta^{-d} |\cX_n \cap B(x, \delta)| \pi_{X_n}(x)\\
&= \lim_{\delta \to 0} V_d^{-1}\delta^{-d} \int_{|y - x| < \delta} \pi_Y(y) dy = \pi_Y(x).
\end{align*}
Combining with Lemma \ref{lem:fokker} yields the desired
\[
\lim_{n\to \infty} n \pi_{X_n}(x) = \frac{\pi_Y(x)}{p(x)} = c \frac{p(x)}{\beps(x)^2}. \qedhere
\]
\end{proof}

\subsection{Extension to isotropic graphs}

To obtain our stationary distribution in Theorem \ref{thm:mainthm} we require only convergence to some It\^o process via the Stroock-Varadhan criterion. We can achieve this under substantially more general conditions. We define a class of neighborhood graphs on $\cX_n$ termed \emph{isotropic} over which we have consistent metric recovery without knowledge of the graph construction method.

\begin{defn}[Isotropic]
A graph edge connection procedure on $\cX_n$ is isotropic if it satisfies:
\begin{description}
\setlength{\itemsep}{-1pt}
\item[Distance kernel:] The probability of placing a directed edge from $i$ to $j$ is defined by a kernel function $h(r_{ij})$ mapping locally scaled distances
\[
r_{ij} = |x_i - x_j|\epsilon_n(x_i)^{-1}
\]
with $\epsilon_n(x)$ obeying ($\star$) to probabilities
\item[Nonzero mass:] The kernel function $h(r)$ has nonzero integral $\int_0^1 h(r) r^{d-1} dr > 0$.
\item[Bounded tails:] For all $r > 1$, $h(r) = 0$.
\item[Continuity:] The scaling $n\pi_{X_n}(x)$ of the stationary distribution is uniformly equicontinuous.
\end{description}
\end{defn}

This class of graph preserves the property that the random graph is entirely determined by the underlying density $p(x)$ and local scale $\beps(x)$; this allows us to have the same tractable form for the stationary distribution.

Both constant $\epsilon$ and $k$-nearest neighbor graphs are isotropic upon assumption of uniform equicontinuity. Another interesting class of graphs allowed by this generalization is truncated Gaussian kernels, where connectivity probability decreases exponentially.  Note that $h(r)$ might not be monotonic or continuous in $r$; one surprising example is $h(r) = 1_{[0.5, 1]}(r)$, which deterministically connects points in an annulus.

\begin{cor}[Generalization] \label{cor:general}
If a neighborhood graph is isotropic, then the limiting stationary distribution follows Theorem \ref{thm:mainthm}, and the density and distances can be estimated by Corollary \ref{cor:maincor}.
\end{cor}
\begin{proof}
\vspace{-1em}
We check the Stroock-Varadhan condition stated in Theorem \ref{thm:stroock}.  For this, we use a version of Lemma \ref{lem:lln} for isotropic graphs, which requires that the ball radius vanishes and that the neighborhood size scales as $\omega(n^{\frac{2}{d+2}} \log(n)^{\frac{d}{d+2}})$.

Vanishing neighborhood radius follows because bounded tails and the fact that the kernel is evaluated on $|x_i - x_j|\eps_n(x_i)^{-1}$ ensure the isotropic graph is a subgraph of the $\eps_n(x)$-ball graph.  Kolmogorov's strong law implies that the stochastic out-degree concentrates around its expectation.  It has the correct scaling because the argument of $h(r)$ is scaled by $\eps_n(x)$. See \prettyref{Sthm:general-degree} for details. Thus the analogue of Lemma \ref{lem:lln} holds.

We then check that the limiting local moments for isotropic graphs are proportional to those of $\eps_n(x)$-ball graphs in \prettyref{Slem:polyint}.  All but one of the conditions for the Stroock-Varadhan criterion follow from this; the last \prettyref{Sthm:f4} follows from the bounded ball structure of the connectivity kernel.

To check that we obtain the same limiting process and stationary measure, note the ratios of integrals in Theorem \ref{thm:coefs} are unchanged in the isotropic setting.  See \prettyref{Slem:polyint} for details.  Recovering the stationary distribution, density, and local scale is then done in the same manner as in the $\eps$-ball setting.
\end{proof}

\section{Distance recovery via paths}

Our results in Theorem \ref{thm:mainthm} give a consistent estimator for the density $p(x)$ and the local scale $\beps(x)$. These two quantities specify up to isometry the latent metric embedding of $\cX$.

In order to reconstruct distances between non-neighbor points we weight the edges of $G_n$ by weights $w_{ij} = \eps_n(x_i)$ and find the shortest paths over this graph, which we call $\overline{G}_n$. The results of Alamgir-von Luxburg \cite[Section 4.1]{alamgir2012shortest} show that in the $k$-nearest neighbor graph case, setting $w_{ij} = \widehat{\eps}_n(x_i)$ for the estimator $\widehat{\eps}_n$ of $\eps_n$ results in consistent recovery of pairwise distances.

In \prettyref{Sthm:dist}, we give a straightforward extension of this approach to show that given any uniformly convergent estimator of $\eps_n(x)$, the shortest path on the weighted graph $\overline{G}_n$ converges to the geodesic distance.  Applying standard metric multidimensional scaling then allows us to embed these distances and recover the latent space up to isometry.


\FloatBarrier

\begin{figure*}[t!]
\vspace{-3mm}
\centering
\begin{minipage}[t]{0.34\linewidth}
\includegraphics[width=0.8\linewidth,page=1]{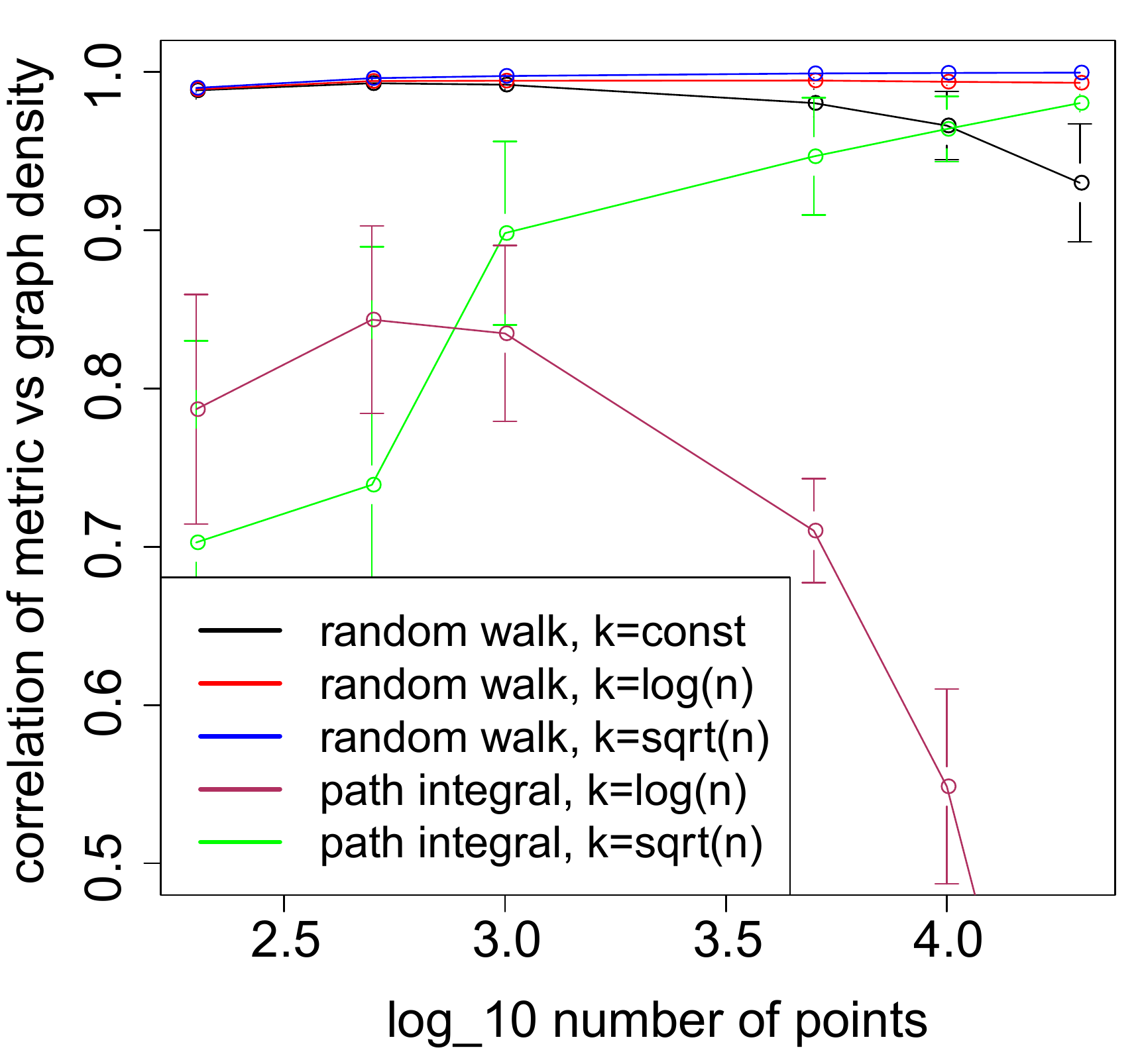}
\caption{Accuracy vs sample and neighborhood size. Path integral (green, maroon) is from Alamgir-von~Luxburg \cite{alamgir_density_2013}. Our estimator (red, blue, black) is nearly perfect at all sample sizes and neighborhood sizes.}
\label{fig2}
\end{minipage}
\hspace{0.5em}
\begin{minipage}[t]{0.63\linewidth}
\includegraphics[width=\linewidth,page=1]{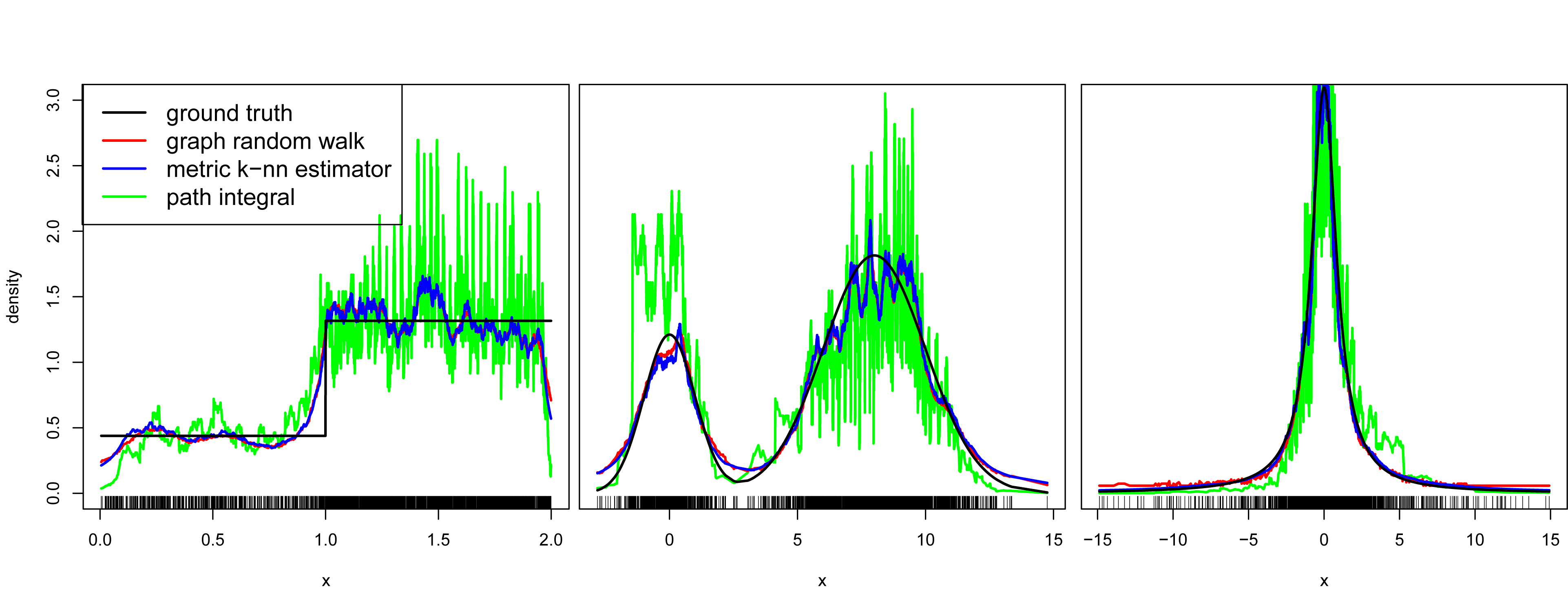}
\caption{Examples of four density estimates: our method (red) using no metric information is indistinguishable from metric $k$-nearest neighbor (blue) and close to ground truth (black). Path integral estimator of Alamgir-von~Luxburg \cite{alamgir_density_2013} (green) shows higher error in all cases.}
\label{fig3}
\end{minipage}
\end{figure*}

\section{Empirical results}

We demonstrate extremely good finite sample performance of our estimator in simulated density reconstruction problems and two real-world datasets. Some details such as exact graph degrees and distribution parameters are in the supplementary code which reproduces all figures in this paper. Standard graph statistics such as centrality and Jaccard index are calculated via the \texttt{igraph} package \cite{igraph}.

\paragraph{$k$-nearest neighbor graphs}

We compared our random-walk based estimator and the path-integral based estimator of von~Luxburg-Alamgir \cite{alamgir_density_2013} to the metric $k$-nearest neighbor density estimator. The number of samples $n$ was varied from $100$ to $20000$ along with the sparsity level $k$ (Figure \ref{fig2}).

While our theoretical results suggest that both our algorithm and the path-integral estimator of von~Luxburg-Alamgir \cite{alamgir_density_2013} might fail to converge at $\sqrt{n}$ and $\log(n)$ sparsity levels, in practice our estimator performs nearly perfectly at both low sparsity levels.

For constant degree $k=50$ we achieve near-perfect performance for all choices of $n$, while the path-integral estimator fails to converge in the $k = \log(n)$ regime.

Some specific examples of our density estimator with $n=2000,k=100$ are shown in Figure \ref{fig3}. The examples are mixture of uniforms (left), mixture of Gaussians (center), and $t$-distribution (right). As predicted, our estimator tracks extremely closely with the metric $k$-nearest neighbor estimator (red and blue), as well as the true density (black).  The path integral estimator has high estimate variance at points with large density and fails to cope with the two mixture densities.

Varying the dimension for an isotropic multivariate normal with $k=\sqrt{n}$, we find that a large number of points are required to maintain high accuracy as $d$ grows large (red and blue lines in Figure \ref{dim}). However, this is due to a global `flattening' of the density. Measuring the correlation between the true and estimated log probabilities show that up to a global concentration parameter, the estimator maintains high accuracy across a large number of dimensions (black lines).

\begin{figure*}[t!]
\vspace{-5mm}
\begin{minipage}[t]{0.3\linewidth}
\centering
\includegraphics[width=0.9\linewidth,page=1]{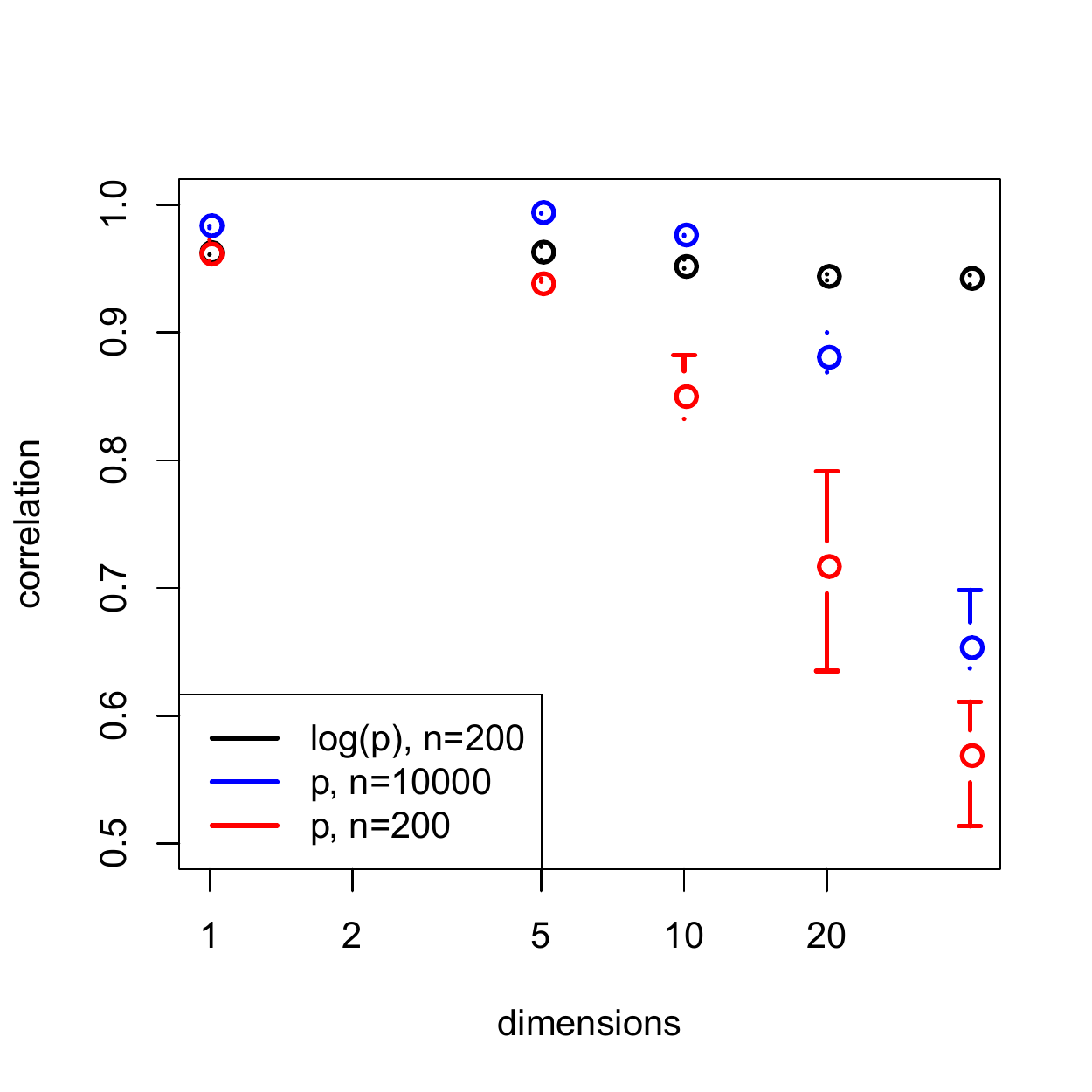}
\caption{Estimate performance degrades in high dimensions due to over-smoothing (blue and red), but the estimator is still highly accurate up to log concentration parameter (black).}
\label{dim}
\end{minipage}
\hspace{1em}
\begin{minipage}[t]{0.7\linewidth}
\includegraphics[width=\linewidth]{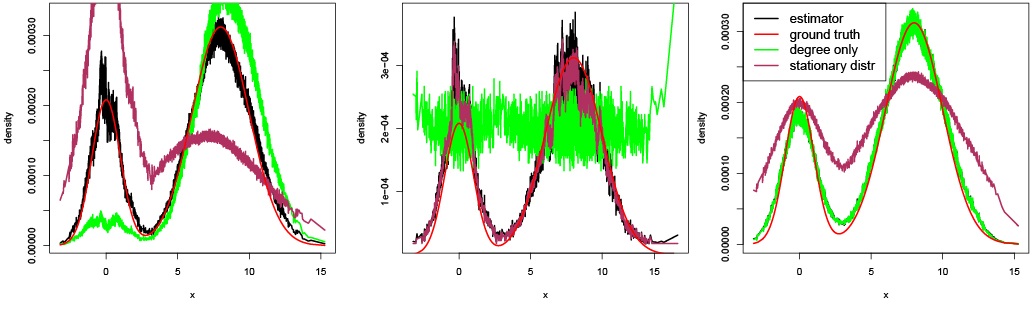}
\caption{Example isotropic graphs. Our estimator (black) agrees with the true density (red) in all cases. Degree and stationary distribution (green and maroon) based density estimates work for some cases (right two panels) but cannot work if the degree is tied to spatial location (left).}
\label{kernel}
\end{minipage}
\end{figure*}

\paragraph{Kernel graphs}

We validate the nonparametric estimator in Corollary \ref{cor:general} for kernel graphs by constructing three drastically different kernel graphs. In all cases, we sampled 5000 points with the connection probability following $p_{i,j} = \exp(-\epsilon(x_i)^{-1}|x_i-x_j|)$. We varied the neighborhood structure $\epsilon$ in  three ways: a constant kernel, $\epsilon(x_i) \propto 1$; $k$-nearest neighbor kernel: $\epsilon(x) \propto 1/\epsilon_{k=100}$; and spatially varying kernel $\epsilon(x) \propto |x|$.

In Figure \ref{kernel}, we find that our nonparametric estimator (black) always matches the ground truth (red). This example also shows that both the degree and the stationary distribution can be valid density estimators under certain assumptions, but only our estimator can deal with arbitrary isotropic graph construction methods without assumptions.

\begin{figure*}[t!]
\centering
\vspace{-3mm}
\begin{minipage}[t]{0.25\textwidth}
\includegraphics[width=0.7\textwidth]{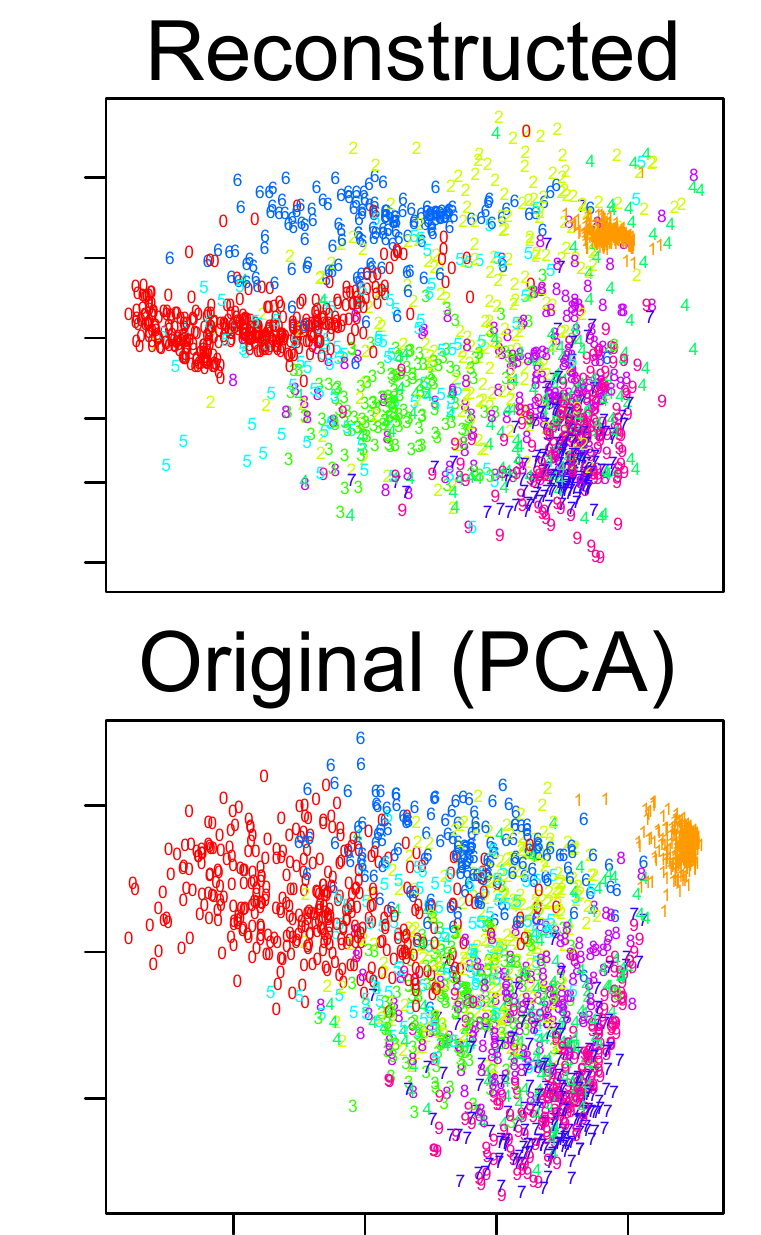}
\caption{Reconstruction closely matches projection of the true metric.}
\label{usps}
\end{minipage}
\hspace{0.5em}
\begin{minipage}[t]{0.27\textwidth}
\includegraphics[width=\textwidth]{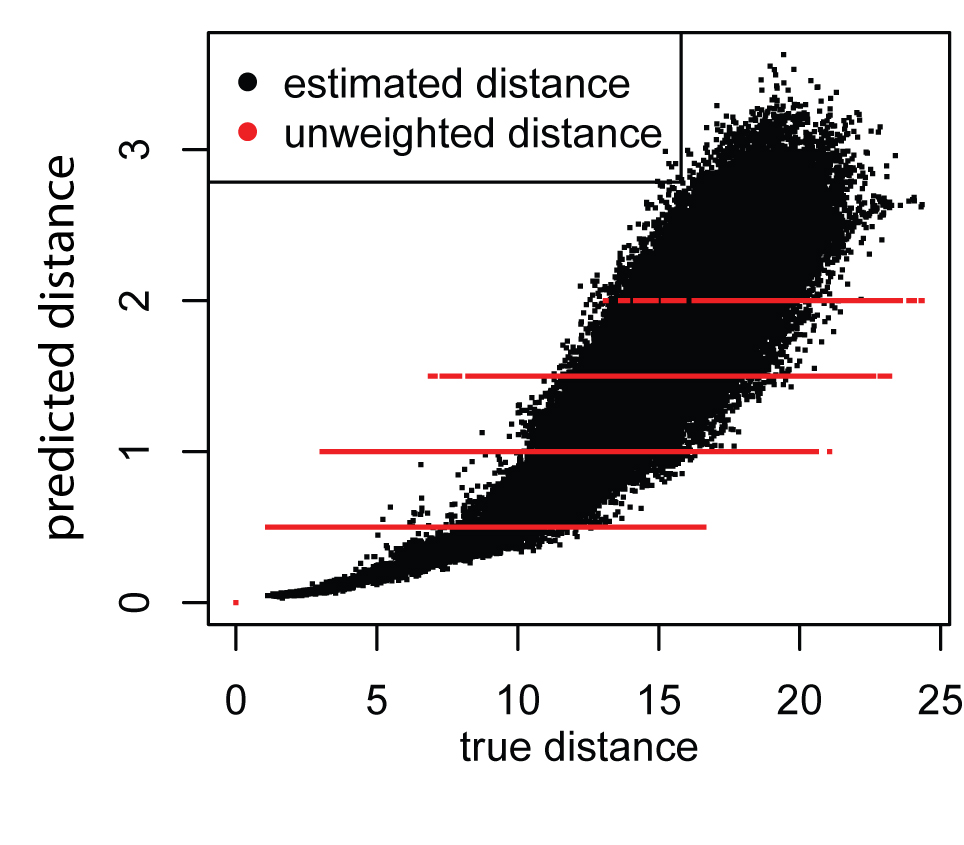}
\caption{Distances estimated by our method are globally close to the true metric.}
\label{dist}
\end{minipage}
\hspace{0.5em}
\begin{minipage}[t]{0.4\textwidth}
\includegraphics[width=\textwidth]{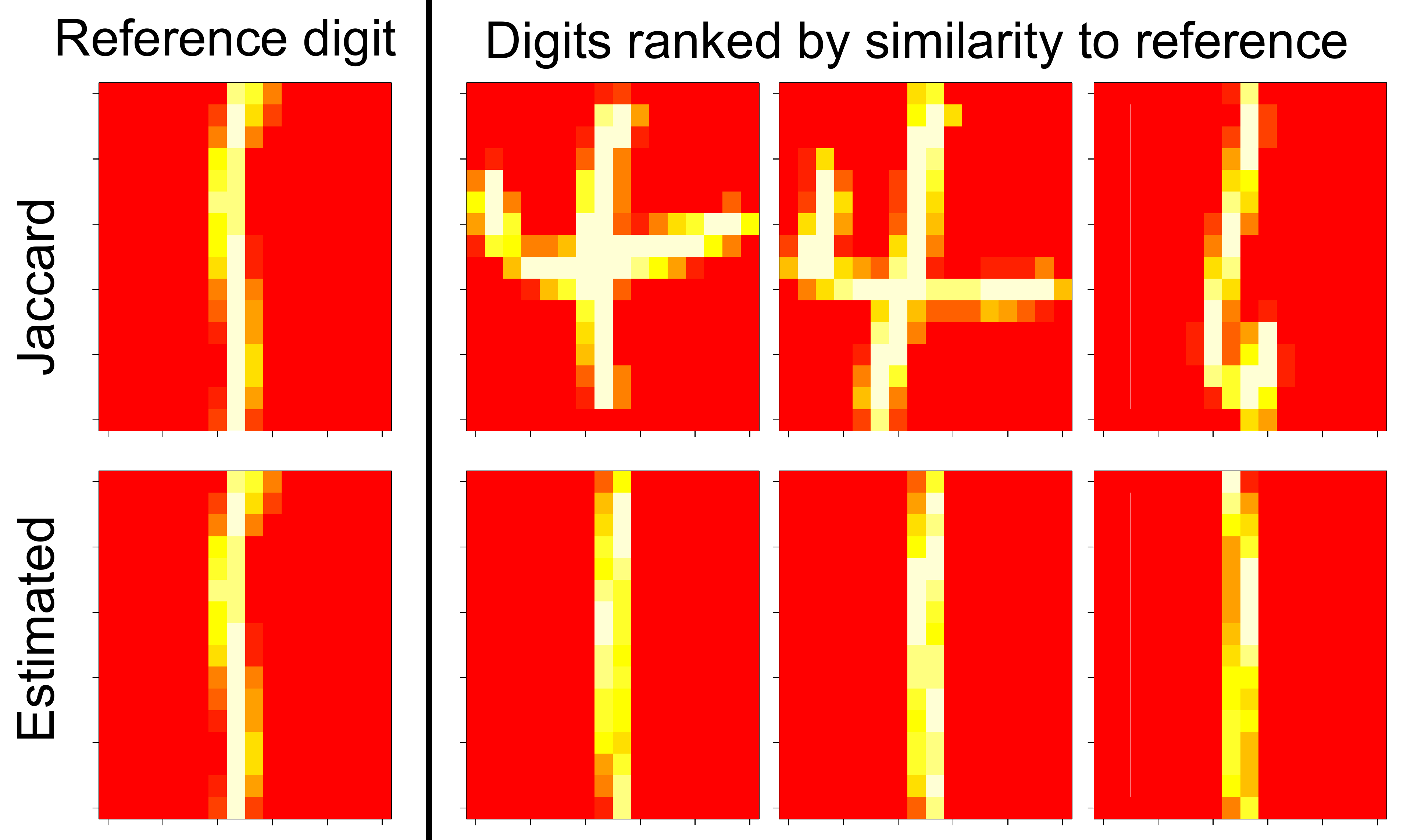}
\caption{Items close in our weighted graph (bottom) are more similar than those under the Jaccard index (top).}
\label{similarity}
\end{minipage}
\end{figure*}

\paragraph{Metric recovery on real data}

As an example of metric reconstruction, we take the first 2000 examples in the U.S. postal service (USPS) digits dataset \cite{hull1994database} and construct an unweighted $k$-nearest neighbor graph. We use our method to reconstruct the metric and perform similarity queries, and the Jaccard index was used to tie-break direct neighbors.

The USPS digits dataset is known to have a high-density cluster of ones digits (orange). Results in Figure \ref{usps} show that we are able to successfully recover the density structure of the data (top). Inter-point distances estimated by our method (Figure \ref{dist}, $y$-axis) show nearly linear agreement to the true metric ($x$-axis) at short distances and high similarity globally.

Performing a similarity query on the data (Figure \ref{similarity}) shows that the our reconstructed distances (bottom row) have a more coherent set of similar digits when compared to the Jaccard index (top row) \cite{jacindex}. The behavior of the unweighted Jaccard similarity is due to a known problem with shortest paths in $k$-nearest neighbor graphs preferring low density regions \cite{alamgir_density_2013}.

\paragraph{Amazon co-purchasing data}

\begin{figure*}[t!]
\vspace{-5mm}
\centering
\begin{minipage}[t]{0.5\textwidth}
\includegraphics[width=0.9\textwidth]{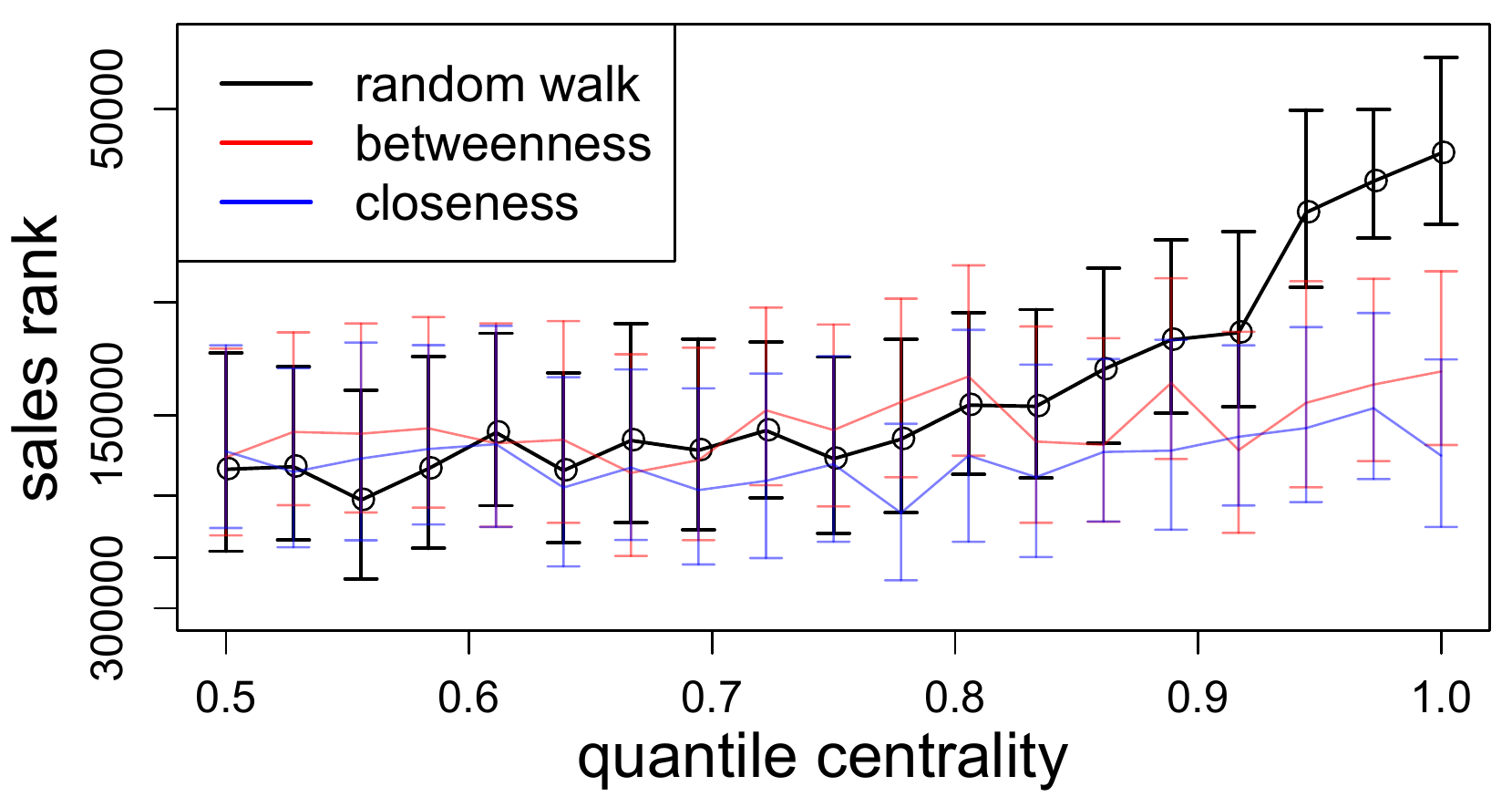}
\caption{Density estimates in the graph correlate well with sales rank, unlike other measures of centrality.}
\label{amazon}
\end{minipage}
\hspace{0.5em}
\begin{minipage}[t]{0.35\textwidth}
\includegraphics[width=\textwidth]{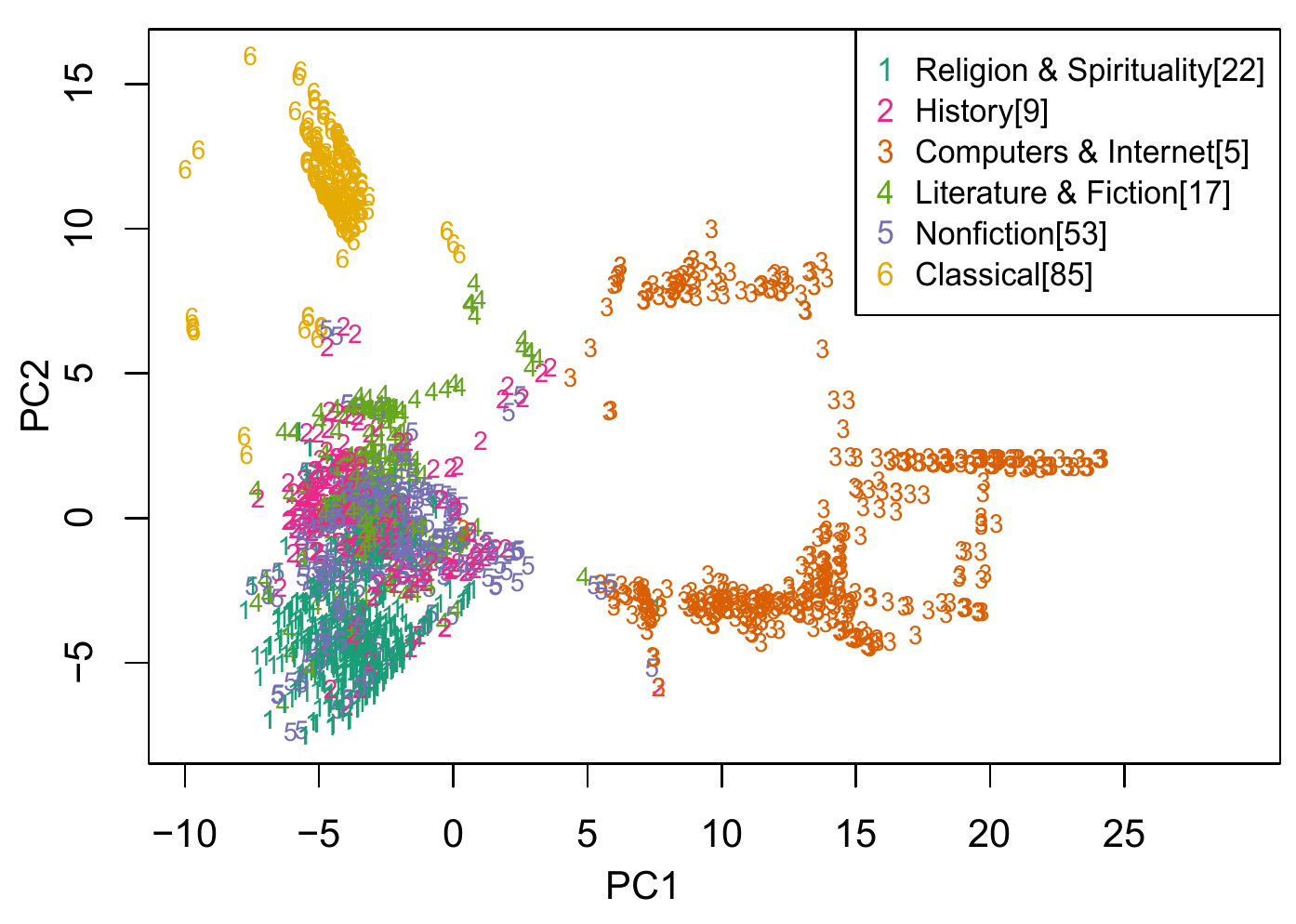}
\caption{Embeddings from estimated distances recover the separation between different product categories.}
\label{amzembed}
\end{minipage}
\end{figure*}

\begin{table*}[ht]
\centering
\vspace{-5mm}
\small
\setlength{\tabcolsep}{3pt}
\resizebox{\textwidth}{!} {
\begin{tabular}{llll}
  \hline
 Classics & Literature & Classical music & Philosophy \\
  \hline
 The Prince & The Stranger & Beethoven: Symphonien Nos. 5 \&  & The Practice of Everyday Life \\
 The Communist Manifesto & The Myth of Sisyphus & Mozart: Symphonies Nos. 35-41 & The Society of the Spectacle \\
 The Republic & The Metamorphosis & Mozart: Violin Concertos & The Production of Space \\
 Wealth of Nations & Heart of Darkness & Tchaikovsky: Concerto No. 1/Rac & Illuminations \\
 On War & The Fall & Beethoven: Symphonies Nos. 3 \&  & Space and Place: The Perspectiv \\
   \hline
\end{tabular}
}
\caption{Top 4 clusters formed by mapping each item to its mode (first row). Each group is a coherent genre.}
\label{cluster}
\end{table*}

Finally, we recover density and metric on a real network dataset with no ground truth. We analyzed the largest connected component of the Amazon co-purchasing network dataset \cite{leskovec2007dynamics}.  Each vertex is a product on \url{amazon.com} along with its category and sales rank, and each directed edge represents a co-purchasing recommendation of the form ``person who bought $x$ also bought $y$.'' This dataset naturally fulfills our assumptions of having edges that are asymmetric, where edges represent a notion of similarity in some space.

The items that lie in regions of highest density should be archetypal products for a category, and therefore be more popular. We show that the density estimates using our method with $d=10$ show a strong positive association between density and sales (Figure \ref{amazon}). We found that this effect persisted regardless of choice of $d$. Other popular measures of network centrality such as betweenness and closeness fail to display this effect.

We then attempted metric recovery using our random walk based reconstruction (Figure \ref{amzembed}). For visualization purposes, we used multidimensional scaling on the recovered metric to embed points belonging to categories with at least two hundred items. The embedding shows that our method captures the separation across different product categories. Notably, nonfiction and history have substantial overlap as expected, while classical music CD's and computer science books have little overlap with the other clusters.

Analyzing the modes of the density estimate by clustering each point to its local mode, we find coherent clusters where top items serve as archetypes for the cluster (Table \ref{cluster}). This suggests that there may be a close connection between clustering in a metric space and community detection in network data.
The overall performance of our method on density estimation and metric recovery for the Amazon dataset suggests that when a metric assumption is appropriate, our random walk based metric quantities can be used directly for centrality and cluster estimates on a network.

\section{Conclusions}

We have presented a simple explicit identity linking the stationary distribution of a random walk on a neighborhood graph to the density and neighborhood size.

The density estimator constructed by inverting this identity shows an extremely rapid convergence to the metric $k$-nearest neighbor density estimator across a range of data point count, sparsity level, and distribution type (Figures \ref{fig2},\ref{fig3}). We also generalized the theorem to a large class of graph construction techniques and demonstrated that the choice of construction technique matters little for accuracy (Figures \ref{kernel}).

Our estimator performed well on real-world data, recovering underlying metric information in test data (Figures \ref{dist},\ref{similarity}) and predicting popular Amazon products through density estimates (Figure \ref{amazon}).

There are several open questions left unanswered by our work. Our results required that the graphs be of degree $k = \omega(n^{2/(d+2)}\log(n)^{d/(d+2)})$ rather than the $\log(n)$ required for connectivity. Our simulation results seem to suggest than even near the $\log(n)$ regime our estimator performs nearly perfectly, suggesting that the true degree lower bound may be much lower. 

The close connection of our density estimate to PageRank suggests that combining the latent spatial map with vector space estimates may lead to highly effective and theoretically principled network algorithms.



\bibliography{knndraft}
\bibliographystyle{abbrv}



\section*{Supplementary material}

# Supplementary proofs for: Metric recovery from directed unweighted graphs

October 26, 2014

# Contents

# 1 Conjecture on uniform equicontinuity of the rescaled stationary distribution

In the conditions ($\star$) we imposed, we required the uniform equicontinuity of $n\pi_{X_n}$. Without this condition, our proof technique implies the weak convergence

$$\sum_{x \in \mathcal{X}_n} \pi_{X_n}(x)\delta_x \to \pi_Y(x)dx$$

of the empirical stationary measures of $X_n(t)$ to the stationary measure of $Y(t)$. The additional imposition of uniform equicontinuity was required solely to upgrade this convergence to a convergence of the rescaled discrete density functions

to the continuous density function. We conjecture that this continuity is true in general.

**Conjecture S1.1.** *Given the other continuity and scaling conditions on $p(x)$ and $\varepsilon_n(x)$ in $(\star)$, $n\pi_{X_n}(x)$ is a.s. uniformly equicontinuous.*

We discuss a few reasons why we might believe this conjecture to hold.

- In the case of constant $\varepsilon_n(x)$, $n\pi_{X_n}(x)$ is proportional to $|\mathsf{NB}_n(x)|$, hence converges to $p(x)$ uniformly. The conjecture therefore holds in this case.

- Our empirical results produce robust results across a broad range of $n$, $\overline{\varepsilon}(x)$, and $p(x)$. One possible explanation would be that Conjecture S1.1 holds for all datasets constructed according to $(\star)$.

- For $x, y \in \mathcal{X}_n$, let $r_n(x)$ denote the expected first return time to $x$ and $c_n(x, y)$ denote the expected commute time from $x$ to $y$. It is known that

$$\pi_{X_n}(x) = \frac{1}{r_n(x)},$$

so to show that $n\pi_{X_n}(x)$ is uniformly equicontinuous, it suffices to show that $\frac{n}{r_n(x)}$ is uniformly equicontinuous. Notice that

$$r_n(x) \leq c_n(x, y) + r_n(y) + c_n(y, x)$$

and that

$$r_n(y) \leq c_n(x, y) + r_n(x) + c_n(y, x),$$

which together imply that

$$|r_n(x) - r_n(y)| \leq |c_n(x, y) + c_n(y, x)|.$$

This relates continuity of $r_n(x)$ and hence $\pi_{X_n}(x)$ to the commute time $c_n(x, y)$. On the other hand, our techniques using the Stroock-Varadhan criterion yield convergence of the simple random walk $X_n(t)$ to the Itô process $Y(t)$ in $\mathsf{D}([0,\infty), \overline{D})$ without assumption of uniform equicontinuity. In a scaling limit, this should lead to a relation between $c_n(x, y)$ and a rescaling of the commute time of the corresponding Itô process. In future work, we intend to use this result to relate a scaling of $c_n(x, y)$ to $|x - y|$ and approach Conjecture S1.1 in conjunction with new methods for metric estimation.

## 2 Full proof of Theorem 2.1

The goal of this section will be to give a fully rigorous proof of Theorem 3.4 from the main text. We first restate the theorem as Theorem S2.1.

**Theorem S2.1.** *Under ($\star$), if $h_n \to g_n^2$ as $n \to \infty$, then a.s. in $\mathcal{X}$, the process $X_n(\lfloor t/h_n \rfloor)$ converges in $\mathsf{D}([0,\infty), \overline{D})$ to the isotropic $\overline{D}$-valued Itô process $Y(t)$ with reflecting boundary condition defined by*

$$dY(t) = \frac{\nabla p(Y(t))}{3p(Y(t))}\overline{\varepsilon}(Y(t))^2 dt + \frac{\overline{\varepsilon}(Y(t))}{\sqrt{3}} dW(t), \tag{1}$$

*where the precise meaning of the reflecting boundary condition is given in Subsection 2.1.*

Our technique is an application of the Stroock-Varadhan criterion (see [2, Theorem 6.3]) for convergence of discrete time Markov processes in a bounded domain to drift-diffusion processes with reflecting boundary conditions in that domain. In what follows, we preserve the notation used by Stroock-Varadhan in [2] whenever possible.

## 2.1 Definition of the objects

In this subsection, we recall in detail the problem setup. We are given an infinite sequence $\mathcal{X} = \{x_1, x_2, \ldots\}$ of latent coordinate points drawn independently from a distribution with probability density $p(x)$ in $\mathbb{R}^d$ supported on a compact domain $D \subset \mathbb{R}^d$ with smooth boundary $\partial D$. We may then find a bounded $C^2$ function $\phi(x)$ on $\mathbb{R}^d$ so that

$$D = \{x \mid \phi(x) > 0\}, \qquad \partial D = \{x \mid \phi(x) = 0\}, \text{and} \qquad |\nabla \phi(x)| \geq 1 \text{ on } \partial D.$$

We fix a single random draw of $\mathcal{X}$ and analyze the quenched setting.

We are then given a radius function $\varepsilon_n(x_i)$ which may depend on the draw of $\mathcal{X}$ and a scaling factor $g_n$ so that

$$\lim_{n \to \infty} g_n^{-1} \varepsilon_n(x) = \overline{\varepsilon}(x)$$

for some deterministic $\overline{\varepsilon}(x)$ on $\overline{D}$. Let $G_n = (\mathcal{X}_n, E_n)$ be the unweighted directed neighborhood graph with vertex set $\mathcal{X}_n = \{x_1, \ldots, x_n\}$ and with a directed edge from $i$ to $j$ if and only if

$$|x_i - x_j| < \varepsilon_n(x_i).$$

Note that $G_n$ is stochastic and depends on the specific realization of $\mathcal{X}_n$ which is drawn.

Let $X_n(t)$ be the simple random walk on the directed graph $G_n$ so that $X_n(t)$ is a discrete-time Markov process with state space $\mathcal{X}_n$. We normalize the timestep of $X_n(t)$ to be $h_n = g_n^2$ and identify $X_n(t)$ with the continuous time process given by $t \mapsto X_n(\lfloor t/h_n \rfloor)$. From now on, we refer to these two processes interchangeably.

In Theorem S2.1, we wish to show that $X_n(t)$ converges weakly in $\mathsf{D}([0,\infty), \overline{D})$ to the continuous-time continuous-space Itô process $Y(t)$ defined by (1) with reflecting boundary conditions. We interpret the boundary conditions in terms of

the submartingale condition of [2]. That is, we define the vector function $\gamma(s, x)$ to be the normal vector to $\partial D$ at $x$ whose length is normalized so that

$$\langle \gamma(s, x), \nabla \phi(x) \rangle = 1.$$

Take also the scalar function $\rho(s, x) = 0$. Together, $\gamma$ and $\rho$ specify the boundary conditions in the following sense.

We say that a process $Y(t)$ solves the submartingale problem for $a$, $b$, $\rho$, and $\gamma$ if for any function $f \in \mathsf{C}_0^{1,2}([0, \infty) \times \overline{D})$ satisfying

$$\rho(\partial f/\partial t) + \langle \gamma, \nabla f \rangle \geq 0$$

on $[0, \infty) \times \partial D$, the random variable

$$f(t, Y(t)) - \int_0^t (f_s + L_s f)(s, Y(s)) \, 1_D(Y(s)) ds$$

is a submartingale, where

$$L_s f = \frac{1}{2} \sum_{i,j=1}^d a^{ij} \frac{\partial^2 f}{\partial x_i \partial x_j} + \sum_{j=1}^d b^j \frac{\partial f}{\partial x_j}.$$

As explained in [2], when $\rho = 0$, this formulation is equivalent to specifying that $Y(t)$ satisfies (1) on the interior of $D$ and has reflecting boundary conditions on $\partial D$.

## 2.2 Quantities used in the Stroock-Varadhan criterion

We now define the moment and boundary quantities which are used in the Stroock-Varadhan criterion. We follow the notations of [2]. Our discrete time Markov process $X_n(t)$ has time increment $h_n = g_n^2$ and transition kernel

$$\Pi_n(x, A) = p(X_n(t+1) \in A | X_n(t) = x) = \frac{|\mathcal{X}_n \cap A \cap B(x, \varepsilon_n(x))|}{|\mathcal{X}_n \cap B(x, \varepsilon_n(x))|}$$

for $x \in \mathcal{X}_n$, where we recall that $\mathcal{X}_n \cap B(x, \varepsilon_n(x)) = \mathsf{NB}_n(x)$.

The moment quantities in [2] are the discrete time drift $b_n$, diffusion $a_n$, and tail $\Delta_{n,\alpha}$ coefficients, defined for $x \in \mathcal{X}_n$ by

$$a_n^{ij}(s, x) = \frac{1}{h_n} \int (y_i - x_i)(y_j - x_j) \Pi_n(x, dy) = \frac{1}{h_n} \sum_{y \in \mathsf{NB}_n(x)} \frac{(y_i - x_i)(y_j - x_j)}{|\mathsf{NB}_n(x)|}$$

$$b_n^i(s, x) = \frac{1}{h_n} \int (y_i - x_i) \Pi_n(x, dy) = \frac{1}{h_n} \sum_{y \in \mathsf{NB}_n(x)} \frac{y_i - x_i}{|\mathsf{NB}_n(x)|}$$

$$\Delta_{n,\alpha}(s, x) = \frac{1}{h_n} \int |y - x|^{2+\alpha} \Pi_n(x, dy) = \frac{1}{h_n} \sum_{y \in \mathsf{NB}_n(x)} \frac{|y - x|^{2+\alpha}}{|\mathsf{NB}_n(x)|}.$$

4

The boundary conditions are specified by $\gamma$ and $\rho$, where we recall that $\rho \equiv 0$. We note that $\gamma$ has the alternate expression

$$\gamma(s, x) = C_\gamma(x) \lim_{n \to \infty} \varepsilon_n(x)^{-1} \int_{|y| < \varepsilon_n(x)} y \frac{p(x+y)}{p_{\varepsilon_n(x)}(x)} dy,$$

where $p_r(x) = \int_{|y|<r} p(x+y) dy$ for a normalization factor $C_\gamma(x)$. Define $J_0 = \{(t, y) : \rho(t, y) = 0\}$ and $J_1$ as its complement. In our setting, $J_0 = \partial D$ and $J_1$ is empty.

**Remark.** In the definitions above, we have included possible time dependence in all quantities to be consistent with the notation of [2]. However, all processes we consider are time-independent, so this dependence will not exist in our case.

## 2.3 Statement of the Stroock-Varadhan criterion

We now state two theorems of Stroock-Varadhan which together imply the convergence of $X_n(t)$ to $Y(t)$ in $\mathsf{D}([0, T], \overline{D})$ for any $T > 0$. These theorems will depend on several conditions which we label A-E and F1-4 and check in the next subsection.

**Remark.** By [3, Theorem 2.8], convergence in $\mathsf{D}([0, T], \overline{D})$ for all $T > 0$ implies convergence $\mathsf{D}([0, \infty), \overline{D})$. Further, by [1, Theorem 4.9.12], this implies weak convergence of the stationary measures of $X_n(t)$ to the stationary measure of $Y(t)$.

The first theorem yields tightness of measures of $X_n(t)$ on Skorokhod space.

**Theorem S2.2** ([2, Theorem 6.1]). *Suppose a discrete time Markov process $X_n(t)$ satisfies the following conditions.*

A. *(bounded tail mass): For some $\alpha > 0$, as $n \to \infty$, we have*

$$\sup_{0 \le t \le T} \sup_{x \in G} \Delta_{n,\alpha}(t, x) \to 0.$$

B. *(all large drifts are reflections): There exists $M$ and $c$ such that for all $n > n_0$, $|b_n(t, x)| > M$ implies $\frac{\langle \nabla \phi(x), b_n(t,x) \rangle}{|b_n(t,x)|} \ge c$.*

C. *(bounded drift outside boundary): For every $\delta > 0$ there exists some $M_\delta < \infty$ such that for all $n > n_0$, $|b_n(t, x)| > M_\delta$ implies $\phi(x) < \delta$.*

D. *(bounded diffusion): There exists $M < \infty$ such that for all $n > n_0$,*

$$\sup_{0 \le t \le T} \sup_{x \in G} ||a_n(t, x)|| \le M,$$

*where $||\cdot||$ denotes the Frobenius norm.*

5

Then, the family of distributions $P_x^n$ induced by $X_n^x(t)$ over trajectories is conditionally compact in $\mathsf{D}([0,T], \overline{D})$. Moreover, any weak limit of these is concentrated on the subset $\mathsf{C}([0,T], \overline{D}) \subset \mathsf{D}([0,T], \overline{D})$.

The next theorem yields convergence of $X_n(t)$ under convergence of the moment quantities and some regularity conditions on the boundary.

**Theorem S2.3** ([2, Theorem 6.3]). *Suppose $X_n(t)$ satisfies the following.*

- E. *(convergence of coefficients): Drift and diffusion coefficients $a_n$ and $b_n$ converge uniformly on compact subsets $K \subset [0,T] \times D$ to some $a$ and $b$.*

- F1. *(reflectivity at absorbing boundary): Given $(t, y) \in J_1$ and $\varepsilon > 0$, there exists $n_0 < \infty$, $\delta_0 > 0$ such that if $|t - s| < \delta_0$, $|x - y| < \delta_0$, $n > n_0$ and $\langle \nabla \phi(x), a_n(s, x) \nabla \phi(x) \rangle < \delta_0$ the following hold:*

$$|a_n(s, x)| < \varepsilon \qquad \text{and} \qquad |b_n(s, x) - \rho^{-1}(t, y) \gamma(t, y)| < \varepsilon.$$

- F2. *(bounded drift under absorption): Given $(t, y) \in J_1$ there exist $M_0 < \infty$ and $\delta_0 > 0$ such that if $|s - t| < \delta_0$ and $|y - x| < \delta_0$, then*

$$|b_n(s, x)| \leq M_0 \quad \text{for all} \quad n.$$

- F3. *(drift dominates diffusion on reflection): Given $(t, y) \in J_0$ and $M < \infty$ there exist $\delta_0 > 0$ and $n_0 < \infty$ such that if $|t - s| < \delta_0, |x - y| < \delta_0, n > n_0$, and $\langle \nabla \phi(x), a_n(s, x) \nabla \phi(x) \rangle < \delta_0$, we have*

$$|b_n(s, x)| \geq M.$$

- F4. *(drifts at boundary simulate reflection): Given $(t, y) \in J_0$ and $\varepsilon > 0$ there exist $\delta_0 > 0, n < \infty$ and $M < \infty$ such that if $|s - t| < \delta_0, |x - y| < \delta_0, n > n_0$, and $|b_n(s, x)| > M$, then*

$$\left| \frac{b_n(s, x)}{\langle b_n(s, x), \nabla \phi(x) \rangle} - \gamma(t, y) \right| < \varepsilon.$$

Then any weak limit $Y(t)$ of $X_n(t)$ in $\mathsf{D}([0,T], \overline{D})$ solves the submartingale problem for $a$, $b$, $\rho$, and $\gamma$.

Finally, we state a criterion for uniqueness of solution to the submartingale problem for $a$, $b$, $\rho$, and $\gamma$.

**Theorem S2.4** ([2, Theorem 5.8]). *Suppose $a$, $b$, $\rho$, and $\gamma$ are time independent and satisfy the following conditions.*

1. *$a$ is continuous, symmetric, and positive definite on $\overline{D}$;*

2. *$b$ is bounded and measurable;*

6

3. $\gamma$ is bounded, locally Lipschitz, and on $\partial D$ satisfies

$$\langle \gamma(x), \nabla \phi(x) \rangle \geq \beta > 0;$$

4. $\rho(x)$ is bounded, continuous, and non-negative.

*Then the solution to the submartingale problem for $a$, $b$, $\rho$, and $\gamma$ is unique.*

Combining Theorem S2.2, Theorem S2.3, and Theorem S2.4 yields the following conclusion.

**Corollary S2.5.** *Suppose that $X_n(t)$ satisfies the conditions of Theorem S2.2, Theorem S2.3, and Theorem S2.4. Then $X_n(t)$ converges to $Y(t)$ in $\mathsf{D}([0,T], \overline{D})$.*

*Proof.* By Theorem S2.2, some subsequential limit of $X_n(t)$ exists. Theorem S2.3 implies that any such limit is a solution to the submartingale problem for $a$, $b$, $\rho$, and $\gamma$, so the uniqueness of Theorem S2.4 yields the desired result. □

## 2.4 Verification of the Stroock-Varadhan conditions

We now verify each of the nine conditions necessary for weak convergence. Conditions F1 and F2 are vacuous because $J_1$ is empty for us. We now verify each of the remaining conditions.

### 2.4.1 Moment conditions

**Theorem S2.6** (Condition A). *As $n \to 0$, we have*

$$\sup_{0 \leq t \leq T} \sup_{x \in D} \Delta_{n,1}(t,x) \to 0.$$

*Specifically, we have*

$$\Delta_{n,1}(s,x) \to \lim_{n \to \infty} \frac{1}{h_n} \int_{|y| < \varepsilon_n(x)} |y|^3 \frac{p(x+y)}{p_{\varepsilon_n(x)}(x)} dy = 0.$$

*Proof.* From Lemma 3.2 with $f(x) = |x|^3$. □

**Theorem S2.7** (Condition E). *The sequences of drift and diffusion coefficients $a_n \to a$ and $b_n \to b$ converge uniformly on compact subsets $K \subset [0,T] \times G$. More specifically, the limiting quantities are*

$$a_n^{ij}(s,x) \to \lim_{n \to \infty} \frac{1}{h_n} \int_{|y| < \varepsilon_n(x)} y_i y_j \frac{p(x+y)}{p_{\varepsilon_n(x)}(x)} dy$$

$$b_n^i(s,x) \to \lim_{n \to \infty} \frac{1}{h_n} \int_{|y| < \varepsilon_n(x)} y_i \frac{p(x+y)}{p_{\varepsilon_n(x)}(x)} dy.$$

*Proof.* From Lemma 3.2 with $f(x) = x$ and $f^{ij}(x) = x_i x_j$. □

### 2.4.2 Boundary conditions

**Theorem S2.8** (Condition C). *For $\delta > 0$, there exists $M_\delta < \infty$ and $n_0$ so that for $n > n_0$, $|b_n(t,x)| > M_\delta$ implies $\phi(x) < \delta$.*

*Proof.* On the compact set $\{\phi(x) \geq \delta\}$, $b_n(t,x)$ converges uniformly by Theorem S2.7 and Theorem S2.14 to $\frac{1}{3}\frac{\nabla p(x)}{p(x)}\bar{\varepsilon}(x)^2$, hence is uniformly bounded on this set. $\square$

**Theorem S2.9** (Condition D). *The diffusion term $a_n$ is uniformly bounded by some $M < \infty$ so that*
$$\sup_{s,x,n} ||a_n(s,x)|| \leq M.$$

*Proof.* By definition the diffusion term
$$a_n^{ij}(s,x) = \frac{1}{h_n} \sum_{y \in \mathsf{NB}_n(x)} \frac{1}{|\mathsf{NB}_n(x)|}(y_i - x_i)(y_j - x_j)$$
is an average of numbers bounded by $\frac{\varepsilon_n(x)^2}{h_n}$. This quantity converges to the bounded function $\bar{\varepsilon}(x)$ as $n \to \infty$, yielding the result. $\square$

**Theorem S2.10** (Condition F3). *Given $(t,y) \in J_0$ and $M < \infty$, there exist $\delta_0 > 0$ and $n_0 < \infty$ so that if $|t - s| < \delta_0$, $|x - y| < \delta_0$, $n > n_0$, and $\langle \nabla \phi(x), a_n(s,x)\nabla \phi(x) \rangle < \delta_0$, then $|b_n(s,x)| \geq M$.*

*Proof.* For any $\delta_1 > 0$, by Lemma 3.2, we may choose $n_0$ large enough so that for all $n > n_0$ and $x \in \mathcal{X}_n$, we have
$$||a_n(s,x) - a(s,x)|| < \delta_1,$$
which implies that
$$\langle \nabla \phi(x), a_n(s,x) \nabla \phi(x) \rangle \geq \left(\frac{1}{3}\bar{\varepsilon}(x)^2 - \delta_1^2\right)|\nabla \phi(x)|^2.$$
Because $\bar{\varepsilon}(y) > 0$, we can choose $\delta_0 > 0$ so that $\bar{\varepsilon}(x)^2$ is uniformly bounded away from 0 on $|x - y| < \delta_0$, hence choosing $\delta_1$ small makes the condition vacuous. $\square$

**Theorem S2.11** (Condition F4). *Given $(t,y) \in J_0$ and $\varepsilon > 0$, there exist $\delta_0 > 0$, $n_0 < \infty$, and $M < \infty$ so that if $|t - s| < \delta_0$, $|x - y| < \delta_0$, $n > n_0$, and $|b_n(s,x)| > M$, then*
$$\left|\frac{b_n(s,x)}{\langle b_n(s,x), \nabla \phi(x) \rangle} - \gamma(t,y)\right| < \varepsilon.$$

*Proof.* For any $\varepsilon > 0$, fix $M > 0$ to be chosen later. Choose $\delta_0$ small enough so that if $|x - y| < 2\delta_0$, we have
$$\left|p(x) - p(y) - \frac{(y-x) \cdot \nabla p(y)}{p(y)}\right| < C_1$$

8

for some uniform $C_1$. By Lemma 3.2 and the fact that $|\nabla\phi(x)| \geq 1$ on $\partial D$ and is continuous, we may choose $n_0$ large enough so that for all $n > n_0$ and $|x - y| < \delta_0$, we have

- $\varepsilon_n(x) < \delta_0$;
- $|\varepsilon_n(x)^2 h_n^{-1} - \bar{\varepsilon}(x)^2| < C_2$ for a uniform $C_2 > 0$;
- $\left|b_n(s,x) - E[b_n(s,x)]\right| < M/2$ for $x \in \mathcal{X}_n$;
- $\left|\frac{b_n(s,x)}{\langle b_n(s,x), \nabla\phi(x)\rangle} - \frac{E[b_n(s,x)]}{\langle E[b_n(s,x)], \nabla\phi(x)\rangle}\right| < \varepsilon/2$ for $x \in \mathcal{X}_n$.

If $|b_n(s,x)| > M$ for $n > n_0$, then

$$\left|E[b_n(s,x)]\right| > M/2.$$

Now, orient the coordinate axes so that the first coordinate axis lies on the normal vector from $x$ to $\partial D$, and let $\tau$ be the distance from $x$ to $\partial D$. In this case, we compute

$$E[b_n^1(s,x)] = h_n^{-1} \int_{z \in B(x,\varepsilon_n(x)) \cap D} (z_1 - x_1) \frac{p(z)}{p_{\varepsilon_n(x)}(x)} dy$$
$$= \frac{\varepsilon_n(x) - \min\{\tau, \varepsilon_n(x)\}}{h_n} + \frac{1}{6}\frac{\partial_1 p(x)}{p(x)}\frac{\varepsilon_n(x)^2 + \tau^2}{h_n} + C_3$$

and for $i > 1$ that

$$E[b_n^i(s,x)] = \frac{1}{6}\frac{\partial_i p(x)}{p(x)}\frac{\varepsilon_n(x)^2}{h_n} + C_4 \qquad (2)$$

for error terms $C_3$ and $C_4$ independent of $n$. Choosing $M$ large enough, we find

$$\tau < (1 - C_5(M))\varepsilon_n(x)$$

for a constant $C_5(M) > 0$ independent of $n$, which implies that

$$E[b_n^1(s,x)] \geq C_5(M)\frac{\varepsilon_n(x)}{h_n} + \frac{1}{6}\frac{\partial_1 p(x)}{p(x)}\frac{\varepsilon_n(x)^2 + \tau^2}{h_n} + C_3. \qquad (3)$$

Now, notice that $\gamma(s,y)$ is a vector purely in the normal direction to $\partial D$ at $y$ normalized so that $\langle \gamma(s,y), \nabla\phi(y)\rangle = 1$. Because the constants $C_3, C_4, C_5(M)$ in (3) and (2) are independent of $n$, all terms in these equations aside from $C_5(M)\frac{\varepsilon_n(x)}{h_n}$ scale to constants as we take $n_0$ and $M$ large, so $\frac{E[b_n(s,x)]}{\langle E[b_n(s,x)], \nabla\phi(x)\rangle}$ becomes arbitrarily close to a vector purely in the normal direction to $\partial D$ from $x$. Choosing $\delta_0$ small enough makes these vectors coincide up to error $\varepsilon/2$, which gives the result when combined with the bound

$$\left|\frac{b_n(s,x)}{\langle b_n(s,x), \nabla\phi(x)\rangle} - \frac{E[b_n(s,x)]}{\langle E[b_n(s,x)], \nabla\phi(x)\rangle}\right| < \varepsilon/2$$

we obtained by taking $n_0$ large. $\square$

**Theorem S2.12** (Condition B). *There exist $M$, $c$, and $n_0$ so that for all $n > n_0$, $|b_n(t,x)| > M$ implies*

$$\frac{\langle \nabla \phi(x), b_n(t,x) \rangle}{|b_n(t,x)|} \geq c.$$

*Proof.* By definition, $\gamma(t,x)$ is uniformly bounded above by some $C_0$. Now, by compactness of $\partial D = J_0$, there exists some $\delta > 0$ so that each $x \in \{\phi(y) < \delta\}$ has a corresponding $x' \in \delta D$ so that the conclusion of Theorem S2.11 applies with $\varepsilon = C_0/2$. Taking $M = M_\delta$ and $n_0$ from Theorem S2.8 for this $\delta$ and applying Theorem S2.11 yields that

$$\frac{\langle \nabla \phi(x), b_n(t,x) \rangle}{|b_n(t,x)|} \geq \frac{2}{C_0}. \qquad \square$$

## 2.5 Completing the proof of Theorem S2.1

By Corollary S2.5, to complete the proof of Theorem S2.1, it suffices for us to compute the limiting terms $a$ and $b$ and to verify the conditions of Theorem S2.4 for uniqueness of the submartingale problem. We begin by computing the limiting $a$ and $b$, for which we will need the following lemma.

**Lemma S2.13.** *For $d \geq 2$, let $B_d(r)$ be the $d$-dimensional ball of radius $r$ and $V_d(r) = V_d r^d$ be its volume. As $r \to 0$, we have*

$$\int_{B_d(r)} x_i^n dx = \begin{cases} 0 & n \text{ odd} \\ \frac{2V_{d-1}}{n+1} r^{n+d} + o(r^{n+d}) & n \text{ even} \end{cases}$$

*and*

$$\int_{B_d(r)} x_i^n x_j^m dx = 0 \text{ if } n \text{ odd}.$$

*Proof.* If $n$ is odd, both claims follow because the integrands are odd functions integrated over symmetric domains. If $n$ is even, for the first claim we compute

$$\int_{B_d(r)} x_i^n dx = \int_{-r}^{r} V_{d-1}(\sqrt{r^2 - x^2}) x^n dx = \frac{2V_{d-1}}{n+1} r^{n+d} + o(r^{n+d}). \qquad \square$$

**Theorem S2.14** (Drift diffusion coefficients). *The limiting integrals for drift and diffusion are*

$$a_n^{ii}(s,x) = \frac{1}{h_n}\left(\frac{1}{3}\varepsilon_n(x)^2 + o(\varepsilon_n(x)^2)\right) \to \frac{1}{3}\bar{\varepsilon}(x)^2$$

$$a_n^{ij}(s,x) = \frac{1}{h_n}\frac{o(\varepsilon_n(x)^{d+2})}{2V_{d-1}p(x)\varepsilon_n(x)^d + o(\varepsilon_n(x)^d)} \to 0$$

$$b_n^i(s,x) = \frac{1}{h_n}\left(\frac{1}{3}\frac{\partial_i p(x)}{p(x)}\varepsilon_n(x)^2 + o(\varepsilon_n(x)^2)\right) \to \frac{\partial_i p(x)}{3p(x)}\bar{\varepsilon}(x)^2$$

$$\Delta_{n,1}(x,s) = \frac{1}{h_n}\left(\frac{\varepsilon_n(x)^{d+4}p(x)V_{d-1} + o(\varepsilon_n(x)^{d+4})}{2V_{d-1}p(x)\varepsilon_n(x)^d + o(\varepsilon_n(x)^d)}\right) \to 0.$$

10

*Proof.* Because $p$ is differentiable on $D$, for any $x \in D$ we have the Taylor expansion
$$p(x+y) = p(x) + y \cdot \nabla p(x) + o(|y|^2)$$
of $p$ at $x$, where the convergence is uniform on compact sets. For $n$ large enough so that the ball of radius $\varepsilon_n(x)$ centered at $x$ lies completely inside $D$, we can substitute this expansion into the definitions of $a_n$ and $b_n$. Using Lemma S2.13 to estimate the resulting expressions yields

$$a_n^{ii}(s,x) = \frac{1}{h_n} \frac{\int_{|y|<\varepsilon_n(x)} y_i^2 p(x) + y_i^2 y \cdot \nabla p(x) + y_i^2 o(|y|^2) dy}{\int_{|y|<\varepsilon_n(x)} p(x) + y \cdot \nabla p(x) + o(|y|^2) dy}$$

$$= \frac{1}{h_n} \frac{\frac{2}{3} V_{d-1} p(x) \varepsilon_n(x)^{d+2} + o(\varepsilon_n(x)^{d+2})}{2 V_{d-1} p(x) \varepsilon_n(x)^d + o(\varepsilon_n(x)^d)}$$

$$= \frac{1}{h_n} \left( \frac{1}{3} \varepsilon_n(x)^2 + o(\varepsilon_n(x)^2) \right)$$

and

$$a_n^{ij}(s,x) = \frac{1}{h_n} \frac{\int_{|y|<\varepsilon_n(x)} y_i y_j p(x) + y_i y_j y \cdot \nabla p(x) + y_i y_j o(|y|^2) dy}{\int_{|y|<\varepsilon_n(x)} p(x) + y \cdot \nabla p(x) + o(|y|^2) dy}$$

$$= \frac{1}{h_n} \frac{o(\varepsilon_n(x)^{d+2})}{2 V_{d-1} p(x) \varepsilon_n(x)^d + o(\varepsilon_n(x)^d)}$$

and

$$b_n^i(s,x) = \frac{1}{h_n} \frac{\int_{|y|<\varepsilon_n(x)} y_i p(x) + y_i y \cdot \nabla p(x) + y_i o(|y|^2) dy}{\int_{|y|<\varepsilon_n(x)} p(x) + y \cdot \nabla p(x) + o(|y|^2) dy}$$

$$= \frac{1}{h_n} \frac{\frac{2}{3} V_{d-1} \frac{\partial_i p(x)}{p(x)} \varepsilon_n(x)^{d+2} + o(\varepsilon_n(x)^{d+2})}{2 V_{d-1} p(x) \varepsilon_n(x)^d + o(\varepsilon_n(x)^d)}$$

$$= \frac{1}{h_n} \left( \frac{1}{3} \frac{\partial_i p(x)}{p(x)} \varepsilon_n(x)^2 + o(\varepsilon_n(x)^2) \right).$$

Defining $S_d(r)$ to be the surface area of a radius $r$ ball in $d$ dimensions, we find

$$\Delta_{n,1}(s,x) = \frac{1}{h_n} \frac{\int_{|y|<\varepsilon_n(x)} |y|^3 p(x) + |y|^3 p(x) + |y|^3 o(|y|^3) dy}{\int_{|y|<\varepsilon_n(x)} p(x) + y \cdot \nabla p(x) + o(|y|^2) dy}$$

$$= \frac{1}{h_n} \frac{\int_0^{\varepsilon_n(x)} r^3 S_d(r) p(x) + o(\varepsilon_n(x)^{d+4})}{2 V_{d-1} p(x) \varepsilon_n(x)^d + o(\varepsilon_n(x)^d)}$$

$$= \frac{1}{h_n} \frac{o(\varepsilon_n(x)^{d+4})}{2 V_{d-1} p(x) \varepsilon_n(x)^d + o(\varepsilon_n(x)^d)}.$$

The result follows by taking the $n \to \infty$ limit in each estimate and recalling that $h_n$ was chosen so that $h_n^{-1} \varepsilon_n(x)^2 \to \bar{\varepsilon}_n(x)^2$ and $h_n^{-1} \varepsilon_n(x)^{2+\alpha} \to 0$. The final convergence is uniform on compact sets because the convergence of the initial Taylor expansion was, each integration estimate preserves uniformity, and the limit $h_n^{-1} \varepsilon_n(x)^2 \to \bar{\varepsilon}(x)^2$ is uniform over all of $D$. □

*Proof of Theorem S2.1.* To prove Theorem S2.1, it remains only to check the conditions of Theorem S2.4. Condition (1) follows because $a(x) = \frac{1}{3}\overline{\varepsilon}(x)^2 \cdot I$ is a continuous multiple of the identity. Condition (2) follows because $b(x) = \frac{1}{3}\frac{\nabla p(x)}{p(x)}\overline{\varepsilon}(x)^2$ is evidently bounded and measurable. For Condition (3), $\gamma$ is evidently bounded, locally Lipschitz because it is a normalized vector normal to the smooth $\partial D$, and $\langle \gamma(x), \nabla \phi(x) \rangle = 1$ by definition. Finally, Condition (4) is evident because $\rho \equiv 0$. □

## 3 Generalizing to isotropic graphs

In this section, we give details on how to generalize our results for $\varepsilon_n(x)$-ball graphs to isotropic graphs. The approach is exactly parallel; we verify the conditions of the Stroock-Varadhan criterion and consider the limiting rescaled stationary distribution. We give in this section the necessary estimates of the minimal degree and the drift and diffusion terms. We first present a technical lemma.

**Lemma S3.1.** *For $d \geq 2$, Let $S_d(r)$ be the d-dimensional shell of radius $r$ and $V_d(r) = C_d r^d$ be its volume. As $r \to 0$, we have*

$$\int_{S_d(r)} x_i^n dx = \begin{cases} 0 & n \text{ odd} \\ \frac{2C_{d-1}}{n+1}(n+d)r^{n+d-1} + o(r^{n+d-1}) & n \text{ even} \end{cases}$$

*and*

$$\int_{S_d(r)} x_i^n x_j^m dx = 0 \text{ if } n \text{ odd}.$$

*Proof.* This follows by differentiating Lemma S2.13. □

Let us now consider an isotropic graph model with kernel function $h(r)$. In particular, this implies that there is an edge from $x_i$ to $x_j$ with probability $h(|x_i - x_j|\varepsilon_n(x_i)^{-1})$ and that

$$\int_0^1 h(r)r^{d-1} dr > 0.$$

We characterize the minimal out-degree in this setting.

**Theorem S3.2** (Minimal out-degree). *For an isotropic graph with kernel $h(r)$ satisfying $(\star)$, we have the almost sure convergence*

$$\varepsilon_n(x)^{-d} \frac{|\mathsf{NB}_n(x)|}{|\mathcal{X}_n \cap B(x, \varepsilon_n(x))|} \to C(h)p(x)$$

*for a constant $C(h)$ independent of $x$ and $n$, which implies that the minimal degree $|\mathsf{NB}_n(x)| = \omega(n^{2/(d+2)} \log(n)^{d/(d+2)})$.*

*Proof.* The out-degree of a vertex is the independent sum of binary variables, each with probability $h(|x_i - x_j|\varepsilon_n(x_i)^{-1})$, so Kolmolgorov's strong law yields

$$\varepsilon_n(x)^{-d} \frac{|\mathsf{NB}_n(x)|}{|\mathcal{X}_n \cap B(x,\varepsilon_n(x))|} \to E\left[\varepsilon_n(x)^{-d} \frac{|\mathsf{NB}_n(x)|}{|\mathcal{X}_n \cap B(x,\varepsilon_n(x))|}\right].$$

Let $y(r,\theta)$ be the radial representation of $y$ and let $C = \frac{2C_{d+1}(n+d)}{n+1}$ be the constants in Lemma S3.1. The desired expected value is the integral

$$\begin{aligned}
E\left[\frac{\varepsilon_n(x)^{-d}|\mathsf{NB}_n(x)|}{|\mathcal{X}_n \cap B(x,\varepsilon_n(x))|}\right] &= \int_{y \in B(x,\varepsilon_n(x))} p(x+y)h(|y|\varepsilon_n^{-1}(x))dy \\
&\sim \varepsilon_n(x)^{-d} \int_{y \in B(x,\varepsilon_n(x))} (p(x) + \nabla p(x) \cdot y)h(|y|\varepsilon_n^{-1}(x))dy \\
&= \varepsilon_n(x)^{-d} \int_0^{\varepsilon_n(x)} \int_{\theta \in S_d(r)} (p(x) + \nabla p(x) \cdot y(r,\theta))h(r)dyd\theta \\
&= Cp(x)\varepsilon_n(x)^{-d} \int_0^{\varepsilon_n(x)} h(r)r^{d-1}dr \\
&\quad + \varepsilon_n(x)^{-d} \int_0^{\varepsilon_n(x)} h(r)r^{d-1} \int_{\theta \in S_d(1)} \nabla p(x) \cdot y(1,\theta)drd\theta.
\end{aligned}$$

The latter term is zero by Lemma S3.1 since it is the integral of the odd function $y(1,\theta)$ over a symmetric domain. Now take the substitution $s = r/\varepsilon_n(x)$ to obtain

$$E\left[\frac{\varepsilon_n(x)^{-d}|\mathsf{NB}_n(x)|}{|\mathcal{X}_n \cap B(x,\varepsilon_n(x))|}\right] = Cp(x) \int_0^1 h(s)s^{d-1}ds.$$

The Kolmogorov strong law provides concentration around this value. Noting that $\varepsilon_n(x)^d = \omega(n^{2/(d+2)} \log(n)^{d/(d+2)})$ gives the asymptotic claim. $\square$

Since Theorem S3.2 guarantees that asymptotically we achieve the necessary minimal number of points, and $h(x)$ is zero for $x > 1$, Lemma 3.2 applies to show the moment conditions in the Stroock-Varadhan criterion. For the boundary conditions, note that C, D, and F3 only require convergence of coefficients in Lemma S3.3 to those in Theorem S2.14. Conditions F4 and B rely on two facts, the uniform convergence of coefficients given by Lemma S3.3, and the asymmetry induced by the boundary (3), the proof of which is parallel to the one given for $\varepsilon$-ball graphs. Therefore, to complete the proof the generalization, it remains only to compute the limiting drift and diffusion coefficients.

**Lemma S3.3** (Polynomial integrals with respect to kernel). *Under the same conditions as Theorem S3.2, for any positive integer $\alpha$ we have*

$$\int_{y \in B(x,\varepsilon_n(x))} y_i^\alpha p(x+y)h(|y|\varepsilon_n^{-1}(x))dy \sim V(h,\alpha) \int_{y \in B(x,\varepsilon_n(x))} y_i^\alpha p(x+y)dy$$

*as $n \to \infty$ for a constant $V(h,\alpha)$ independent of $n$ with $V(h,1) = V(h,2)$.*

*Proof.* Perform the same Taylor approximation and radial decomposition as in Theorem S3.2 to obtain

$$\int_{y \in B(x, \varepsilon_n(x))} y_i^\alpha p(x+y) h(|y|\varepsilon_n^{-1}(x)) dy$$
$$\sim \int_0^{\varepsilon_n(x)} \int_{\theta \in S_d(r)} y_i(r, \theta)^\alpha (p(x) + \nabla p(x) \cdot y(r, \theta)) h(r \varepsilon_n^{-1}(x)) dr d\theta.$$

For $\alpha$ an odd integer, by Lemma S3.1 we have

$$\int_{y \in B(x, \varepsilon_n(x))} y_i^\alpha p(x+y) h(|y|\varepsilon_n^{-1}(x)) dy$$
$$\sim \int_0^{\varepsilon_n(x)} h(r\varepsilon_n^{-1}(x)) r^{\alpha+d} \int_{\theta \in S_d(1)} y_i(1, \theta)^\alpha \nabla p(x) \cdot y(1, \theta) dr d\theta$$
$$= \partial p_i(x) \int_0^1 h(r) r^{\alpha+d} dr \varepsilon_n(x)^{\alpha+d} \int_{\theta \in S_d(1)} y_i(1, \theta)^{\alpha+1} d\theta$$
$$\sim V(h, \alpha) \int_{y \in B(x, \varepsilon_n(x))} y_i^\alpha p(x+y) dy$$

for

$$V(h, \alpha) = (\alpha + d + 1) \int_0^1 h(r) r^{\alpha+d} dr.$$

If $\alpha$ is an even integer, we have

$$\int_{y \in B(x, \varepsilon_n(x))} y_i^\alpha p(x+y) h(|y|\varepsilon_n^{-1}(x)) dy$$
$$\sim \int_0^{\varepsilon_n(x)} h(r\varepsilon_n^{-1}(x)) r^{\alpha+d-1} \int_{\theta \in S_d(1)} y_i(1, \theta)^\alpha p(x) dr d\theta$$
$$= p(x) \int_0^{\varepsilon_n(x)} h(r\varepsilon_n^{-1}(x)) r^{\alpha+d-1} dr \int_{\theta \in S_d(1)} y_i(1, \theta)^\alpha d\theta$$
$$= p(x) \varepsilon_n(x)^{\alpha+d} \int_0^1 h(r) r^{\alpha+d-1} dr \int_{\theta \in S_d(1)} y_i(1, \theta)^\alpha d\theta$$
$$\sim V(h, \alpha) \int_{y \in B(x, \varepsilon_n(x))} y_i^\alpha p(x+y) dy$$

for

$$V(h, \alpha) = (\alpha + d) \int_0^1 h(r) r^{\alpha+d-1} dr. \qquad \square$$

The limits of drift and diffusion terms in Theorem S2.14 depend only on ratios of these integrals for $\alpha = 1, 2$, so applying Lemma S3.3 shows that the limits for isotropic graphs are identical to the ones for $\varepsilon$-ball graphs. The remainder of the analysis proceeds unchanged.

14

# 4 Recovery of distances via ball-radii

We will prove that given the ball radii $\varepsilon_n(x_i)$, we can recover point-to-point distances if $x_i$ are located in a convex domain. Otherwise, we recover the geodesic distances. Our goal is to show that for any points $x_i$ and $x_j$, the weighted shortest path distance $d_{ij}$ between the points on the graph $\overline{G}_n$ where outgoing edges are weighted by $\varepsilon_n(x_i)$ converges to the distance $|x_i - x_j|$.

## 4.1 Outline of proof approach

We proceed in two steps. First, we consider the case when $\varepsilon_n(x_i)$ is known exactly. In this case, the weighted shortest path is an upper bound on the true distance. We bound its weighted distance $d_{ij}$ by constructing a path whose weighted distance is close to the geodesic distance.

To control the upper bound, we show that there exists a $\delta$ that converges to zero faster than $\min_{x_i} \varepsilon_n(x_i)$ while still guaranteeing that every ball of size $\delta$ in the domain contains at least one point. Once we find such a $\delta$, the upper bound will follow. Indeed, if we are at some $x$, we can always find a point that whose distance from our target $x_j$ is smaller by at least $\varepsilon_n(x) - \delta$. This gives an upper bound on the number of steps in our path and therefore the total error.

Second, we assume that we are given noisy estimates of $\overline{\varepsilon}(x)$ from our algorithm via the stationary distribution. We use uniform convergence of $\overline{\varepsilon}(x)$ to control the overall pathwise error.

We give a detailed analysis of each step in separate subsections below.

## 4.2 The case of exact knowledge of $\varepsilon_n$.

We begin with two lemmas allowing us to construct a for each pair of points $i, j$ a point $k$ along which to start a path from $i$ to $j$.

**Lemma S4.1.** *Let $\delta_n = \Omega(n^{-\frac{1}{d+1}})$. For any set of $n^2$ balls with radius $\delta_n$, all $n^2$ balls will have at least one point of $\mathcal{X}_n$ with high probability.*

*Proof.* The number of points $N(x)$ in a ball of radius $\delta_n$ follows a binomial distribution with $n$ draws and success probability

$$p_{\delta_n}(x) = \int_{|y-x|<\delta(n)} p(y) dy \sim V_d p(x) \delta_n^d.$$

Therefore, the probability that $N(x) = 0$ is

$$P(N(x) = 0) = (1 - p_{\delta_n}(x))^n = \left((1 - p_{\delta_n}(x))^{p_{\delta_n}(x)^{-1}}\right)^{np_{\delta_n}(x)} \to e^{-np_{\delta_n}(x)}$$

if $n\delta_n^d \to \infty$. Recalling that $\delta_n = \Omega(n^{-\frac{1}{d+1}})$, this implies that

$$np_{\delta_n}(x) \sim n^{\frac{1}{d+1}}$$

and in particular that $P(N(x) = 0) = o(n^{-2})$, so taking the union bound over all $n^2$ balls yields the result. $\square$

15

**Lemma S4.2.** *Let $\delta_n = \Omega(n^{-\frac{1}{d+1}})$. For all $i, j$, there exists $x_k \in B(x_i, \varepsilon_n(x_i))$ such that*

$$\left|\Big(|x_i - x_j| - |x_k - x_j|\Big) - |x_i - x_k|\right| \leq 2\delta_n \text{ and } \Big||x_k - x_i| - \varepsilon_n(x_i)\Big| \leq 2\delta_n.$$

*Proof.* Let $v = \frac{x_j - x_i}{|x_j - x_i|}$ and consider the $n^2$ balls

$$B_{ij} = B(x_i + v(\varepsilon_n(x_i) - \delta_n), \delta_n).$$

By Lemma S4.1, there must exist with high probability at least one point of $\mathcal{X}_n$ in each $B_{ij}$. Any such $x_k \in B_{ij}$ verifies the desired conditions. □

**Theorem S4.3.** *Let $x_i, x_j \in \mathcal{X}_n$ and $d_{ij}$ be the weighted shortest path distance over the weighted graph $\overline{G}_n$ constructed from $G_n$ by assigning weight $\varepsilon_n(x_i)$ to all outgoing edges from $x_i$. For any $\varepsilon > 0$, there exists an $n$ such that*

$$\Big||x_i - x_j| - d_{ij}\Big| < \varepsilon.$$

*Proof.* Take $\delta_n = \Theta(n^{-\frac{1}{d+1}})$. We show that with high probability, there exists a path with $M$ steps whose weighted path distance $d$ satisfies

$$|x_i - x_j| \leq d \leq |x_i - x_j| + 2M\delta_n + \max_{x \in \mathcal{X}_n} \varepsilon_n(x)$$

and so that $\lim_{n \to \infty} M\delta_n = 0$. The result then follows because $d_{ij} \leq d$.

To construct such a path from $x_i$ to $x_j$, we apply the following procedure. Start at the point $x_i$. If the current point is $x_k$ and $x_j \in B(x_k, \varepsilon_n(x_k))$, move to it and terminate. Otherwise, pick a point $x_l \in B_{kj}$ and repeat until $x_j$ is reached.

The lower bound holds because each edge weight is at least its length. For the upper bound, by Lemma S4.2, moving to $x_l$ reduces the geodesic distance to $x_j$ by at least $|x_k - x_l| - 2\delta_n$ and moves a weighted distance of $\varepsilon_n(x_k) < |x_k - x_l| + 2\delta_n$. Thus, if our path has $M$ steps, the difference between our weighted distance and the geodesic distance is at most $4M\delta_n + \max_x \varepsilon_n(x)$, where we add the weighted distance of the last step. This gives the upper bound.

It remains now to bound $M$. For this, notice that the geodesic distance to $x_j$ decreases by at least $\min_{x \in \mathcal{X}_n} \varepsilon_n(x) - 2\delta_n$ at each step, leading to the bound

$$M \leq \frac{|x_i - x_j|}{\min_{x \in \mathcal{X}_n} \varepsilon_n(x) - 2\delta_n}.$$

Recall now that $\delta_n = \Theta(n^{-\frac{1}{d+1}})$ so that $\varepsilon_n(x) = \omega(\delta_n)$ and hence

$$M\delta_n = \frac{|x_i - x_j|}{\min_{x \in \mathcal{X}_n} \frac{\varepsilon_n(x)}{\delta_n} - 1} \to 0. \quad \square$$

## 4.3 The case of stochastic estimates of $\varepsilon_n$

We now consider the case where we are given only an estimate $\widehat{\varepsilon}_n(x)$ of $\varepsilon_n$, obtained by first estimating $\overline{\varepsilon}(x)$ via the stationary distribution and then applying a normalization to obtain $\widehat{\varepsilon}_n(x)$ on $\mathcal{X}_n$. We first control the error in $\widehat{\varepsilon}_n(x)$ along a single path.

**Lemma S4.4.** *For $k_1 = i$ and $k_{l_n} = j$, let $x_{k_1}, \ldots, x_{k_{l_n}}$ be a path between $i$ and $j$ in $\overline{G}_n$. If $l_n = O(g_n^{-1})$, we have*

$$\sum_{i=1}^{l_n} |\widehat{\varepsilon}_n(x_{k_i}) - \varepsilon_n(x_{k_i})| \to 0$$

*in probability.*

*Proof.* By uniform convergence of the stationary distribution and continuity of the out degree estimate $p(x)\varepsilon_n(x)^d V_d = k/n$, for all $\gamma$ and $\delta$, we have

$$P\left(\sup_{x \in \mathcal{X}_n} \left|\frac{\widehat{\varepsilon}_n(x)}{g_n} - \overline{\varepsilon}(x)\right| > \gamma\right) < \delta$$

for large enough $n$. This implies that

$$P\left(\sup_{x \in \mathcal{X}_n} |\widehat{\varepsilon}_n(x) - \varepsilon_n(x)| > \gamma g_n\right) < \delta.$$

Now notice that

$$P\left(\sum_{i=1}^{l_n} |\widehat{\varepsilon}_n(x_{k_i}) - \varepsilon_n(x_{k_i})| > \gamma\right) < P\left(l_n \sup_x |\widehat{\varepsilon}_n(x_{k_i}) - \varepsilon_n(x_{k_i})| > \gamma\right).$$

By assumption, the number of steps in the path is $l_n = O(g_n^{-1})$. Therefore, there exists a constant $M > 0$ such that

$$P\left(\sum_{i=1}^{l_n} |\widehat{\varepsilon}_n(x_{k_i}) - \varepsilon_n(x_{k_i})| > \gamma\right) < P\left(\sup_x |\widehat{\varepsilon}_n(x) - \varepsilon_n(x)| > M\gamma g_n\right) < \delta,$$

from which the claim follows by choosing $n$ large enough. $\square$

We now show that the shortest weighted distance path recovers the geodesic distance with stochastic estimates $\widehat{\varepsilon}_n(x)$ instead of the true values. Our approach is the same as in the deterministic case; we will construct a weighted path and show that its weighted distance converges to the geodesic distance and is close to the weighted distance of the shortest weighted path. Let $\widehat{d}_{ij}$ denote the weighted distance of the shortest weighted distance path from $x_i$ to $x_j$.

**Theorem S4.5.** *For any $\varepsilon > 0$, there exists $n$ such that*

$$\left||x_i - x_j| - \widehat{d}_{ij}\right| < \varepsilon$$

*with high probability.*

*Proof.* Let $\delta_n = \Theta(n^{-\frac{1}{d+2}})$. For any $\gamma > 0$, we show that for large enough $n$, with high probability there exists a path from $x_i$ to $x_j$ with $M$ steps whose weighted path distance $\widehat{d}$ satisfies

$$\widehat{d} \leq |x_i - x_j| + 4M\delta_n + \gamma + \max_{x \in \mathcal{X}_n} \widehat{\varepsilon}_n(x). \tag{4}$$

Construct the path as in Theorem S4.3 with $\varepsilon_n(x)$ replaced by $\widehat{\varepsilon}_n(x)$.

We now analyze its weighted distance. Arguing as in Lemma S4.2, a step from $x_k$ to $x_l$ which is not the last step in this path reduces the geodesic distance to $x_j$ by between $|x_k - x_l| - 2\delta_n$ and $|x_k - x_l|$. On the other hand, this step has a weighted distance of $\widehat{\varepsilon}_n(x_k)$, which satisfies

$$|x_k - x_l| - 2\delta_n - |\widehat{\varepsilon}_n(x_k) - \varepsilon_n(x_k)| \leq \widehat{\varepsilon}_n(x_k) \leq |x_k - x_l| + |\widehat{\varepsilon}_n(x_k) - \varepsilon_n(x_k)|.$$

Therefore, the geodesic distance traveled and weighted distance $\widehat{d}$ along our constructed path differ by at most

$$\sum_{i=1}^{M-1} |\widehat{\varepsilon}_n(x_{k_i}) - \varepsilon_n(x_{k_{i+1}})| + 4M\delta_n$$

in the first $M - 1$ steps. By arguing as in the proof of Theorem S4.3 with $\varepsilon_n(x)$ replaced by $\widehat{\varepsilon}_n(x)$ and noting that $\widehat{\varepsilon}_n(x)$ converges uniformly to $\varepsilon_n(x)$, the number of steps in the constructed path satisfies

$$\frac{|x_i - x_j|}{\max_x \varepsilon_n(x)} \leq M \leq \frac{|x_i - x_j|}{\min_x \varepsilon_n(x) - 2\delta_n}. \tag{5}$$

In particular, we note that $M = O(g_n^{-1})$. Applying Lemma S4.4 to choose $n$ large enough so that

$$\sum_{i=1}^{M-1} |\widehat{\varepsilon}_n(x_{k_i}) - \varepsilon_n(x_{k_{i+1}})| < \gamma$$

and adding $\max_{x \in \mathcal{X}_n} \widehat{\varepsilon}_n(x)$ for the last step yields (4). Noting by (5) that $M\delta_n \to 0$, taking large enough $n$ in (4) shows that $\widehat{d}_{ij} \leq \widehat{d} \leq |x_i - x_j|$.

We now show that $\widehat{d}_{ij} \geq |x_i - x_j|$. It suffices to show that the length $L$ of the shortest weighted distance path must be $L = O(g_n^{-1})$, as Lemma S4.4 would then imply that its weighted distance with respect to $\widehat{\varepsilon}_n(x)$ converges to its weighted distance with respect to $\varepsilon_n(x)$, which is bounded below by $|x_i - x_j|$.

To bound $L$, note that the minimum weighted distance at each step is $\min_{x \in \mathcal{X}_n} \widehat{\varepsilon}_n(x)$, while the total weighed distance is at most $\widehat{d}$. Therefore, by (4), we obtain that for any $\gamma > 0$ we have

$$L \min_{x \in \mathcal{X}_n} \widehat{\varepsilon}_n(x) \leq |x_i - x_j| + 4M\delta_n + \gamma + \max_{x \in \mathcal{X}_n} \widehat{\varepsilon}_n(x)$$

for large enough $n$. By uniform convergence of $\widehat{\varepsilon}_n(x)$ to $\varepsilon_n(x)$, this shows that for any $\gamma > 0$ we have

$$L \leq \frac{|x_i - x_j| + 4M\delta_n + \gamma + \max_{x \in \mathcal{X}_n} \widehat{\varepsilon}_n(x)}{\min_{x \in \mathcal{X}_n} \varepsilon_n(x)} = O(g_n^{-1})$$

for large enough $n$, yielding the desired. $\square$



\end{document}